\documentclass[11pt]{article}
\usepackage{url}
\usepackage[margin=1in]{geometry}
\usepackage{amsmath,amsfonts,amsthm,amssymb,amsopn,bm,color,mathtools,booktabs,multirow,appendix,caption}
\usepackage{dsfont}
\usepackage{algorithm}
\usepackage{algorithmic}
\usepackage{hyperref}
\usepackage{nicefrac}
\usepackage{enumerate}
\usepackage{enumitem}
\usepackage{comment}
\usepackage{microtype}
\usepackage{float}
\usepackage{natbib}
\usepackage{lipsum,lmodern}
\usepackage{todonotes}
\usepackage{fancyvrb}
\usepackage{xspace}
\usepackage{subcaption}
\usepackage{wrapfig}
\usepackage[most]{tcolorbox}
\allowdisplaybreaks

\usepackage{mathtools}  
\usepackage{microtype}  
\allowdisplaybreaks 

\usepackage{thmtools}
\usepackage{thm-restate}

\definecolor{ForestGreen}{rgb}{0.1333,0.5451,0.1333}
\hypersetup{
    colorlinks=true,
    citecolor=ForestGreen,
    linkcolor=blue,
    filecolor=magenta,      
    urlcolor=blue,
}

\def\E{\mathbb{E}}
\def\P{\mathbb{P}}
\def\R{\mathbb{R}}
\def\X{\mathcal{X}}
\def\V{\mathcal{V}}
\def\I{\mathcal{I}}
\def\Y{\mathcal{Y}}
\def\N{\mathbb{N}}
\def\1{\bm{1}}

\newcommand{\mc}[1]{\mathcal{#1}}
\newcommand{\htheta}{\widehat{\theta}}
\newcommand{\otheta}{\overline{\theta}}

\newcommand{\adj}{\mathrm{Adjacent}}

\newcommand{\Sp}[1]{\left(#1\right)}

\newcommand{\Bp}[1]{\left\{#1\right\}}

\newcommand{\Norm}[1]{\left\|#1\right\|}

\DeclareMathOperator*{\argmin}{argmin}
\DeclareMathOperator*{\argmax}{argmax}

\DeclareMathOperator*{\conv}{conv}

\newcommand{\algoname}{\textsf{Adjacent-BAI}\xspace}

\newtheorem{lemma}{Lemma}

\newtheorem{theorem}{Theorem}
\newtheorem{proposition}{Proposition}

\theoremstyle{definition}

\newtheorem{definition}{Definition}

\theoremstyle{remark}

\title{On the Complexity of Best-Arm Identification in Non-Stationary Linear Bandits}
\author{
  Leo Maynard-Zhang\thanks{Equal contribution.} \\ University of Washington \\ \texttt{leomayn@cs.washington.edu}
  \and
  Zhihan Xiong\footnotemark[1] \\ University of Washington \\ \texttt{zhihanx@cs.washington.edu}
  \and
  Kevin Jamieson \\ University of Washington \\ \texttt{jamieson@cs.washington.edu}
  \and
  Maryam Fazel \\ University of Washington \\ \texttt{mfazel@uw.edu}
}
\date{}

\begin{document}
\maketitle

\begin{abstract}We study the fixed-budget best-arm identification (BAI) problem in non-stationary linear bandits.  Concretely, given a fixed time budget $T\in \mathbb{N}$, finite arm set $\mathcal{X} \subset \mathbb{R}^d$, and a potentially adversarial sequence of unknown parameters $\lbrace \theta_t\rbrace_{t=1}^{T}$ (hence non-stationary), a learner aims to identify the arm with the largest cumulative reward $x_* = \arg\max_{x \in \mathcal{X}} x^\top\sum_{t=1}^T \theta_t$ with high probability. In this setting, it is well-known that i.i.d. sampling arms from the G-optimal design yields a minimax-optimal error probability of $\exp\left(-\Theta\left(T /  H_{G}\right)\right)$, where $H_{G}$ scales proportionally with the dimension $d$. However, this notion of complexity is overly pessimistic, as it is derived from a lower bound in which the arm set consists only of the standard basis vectors, thus masking any potential advantages arising from arm sets with richer geometric structure. To address this, we establish an \textit{arm-set-dependent} lower bound that, in contrast, holds for any arm set. Motivated by the ideas underlying our lower bound, we propose the \textit{Adjacent-optimal design}, a specialization of the well-known $\mathcal{XY}$-optimal design, and develop the \textsf{Adjacent-BAI} algorithm. We prove that the error probability of \textsf{Adjacent-BAI} matches our lower bound up to constants, verifying the tightness of our lower bound, and establishing the arm-set-dependent complexity of this setting.

\end{abstract}

\section{Introduction} \label{sec:intro}
 We focus on quantifying the difficulty of fixed-budget best-arm identification (BAI) problem in non-stationary linear bandits. The linear bandit framework generalizes the classical multi-armed bandit problem by associating each arm with a feature vector and modeling rewards as linear in an unknown parameter. Algorithms such as UCB and Thompson Sampling are known to perform optimally in the well-studied regret-minimization setting \citep{lattimore2020bandit}. However, in the best-arm identification setting, where exploration is more important than cumulative reward, these algorithms are suboptimal \citep{bubeck2009pure}, thus requiring specialized attention. In the standard BAI setting, a stationary environment is assumed where the rewards of arms are sampled i.i.d \citep{audibert2010best, soare2014best, NEURIPS2019_8ba6c657}. However, algorithms designed for these environments can completely fail as soon as this assumption is lifted, for example, in settings where the value of an arm can change at any time step \citep{abbasi2018best,xiong2024b}. An interesting question is \emph{exactly how the difficulty of this setting changes once this stationary assumption is lifted}. 

In recent work, \citet{xiong2024b} partially answered this question by showing that when the arm set is restricted to the standard basis, the difficulty of non-stationary BAI scales proportionally with the dimension. While informative, this result is ultimately unsatisfying as it collapses the linear bandit model back into a multi-armed bandit. Basis arm sets erase correlations between arms, precisely the geometric structure that linear bandits are meant to exploit. With this in mind, we aim to uncover precisely which relationship between arms truly governs the difficulty of non-stationary BAI. Inspired by \citet{soare2014best}, a natural conjecture is that the pairwise relationships between every arm should be considered when measuring difficulty. However, such a characterization implies that any arm set containing a basis is no easier than the basis itself, and thus no progress is made. To refine this view, we introduce adjacency, a geometric notion that captures which pairs of arms can compete for optimality. Accordingly, we find that the pairwise relationships between adjacent arms are necessary and sufficient to measure the difficulty of non-stationary BAI. Formally, by leveraging the adjacent structure of the arm set, we refine the previous notion of difficulty in non-stationary BAI by establishing a measure of complexity that adapts to the geometry of \textit{any} arm set, which we term as \textit{arm-set-dependent}.

\paragraph{Our contributions.}
\begin{itemize}
    \item \textbf{Adjacency and Non-Stationary BAI.} The main novelty of our paper stems from our characterization of non-stationary BAI as a problem of differentiating only between adjacent arms, a notion we formalize in later sections. The intuition for this restriction stems from our central Lemma \ref{lem:adjacency_optimal}, which implies that if an arm is better than its adjacent arms, then it is the optimal arm. Accordingly, we introduce a new arm-set-dependent complexity measure $ H_{\mathrm{Adjacent}}(\X)$ and show it strictly refines the minimax-optimal complexity $ H_G$. 
    \item \textbf{Arm-set-dependent lower bound.} We show in Theorem \ref{thm:finallower} that for fixed arm set $\X$, the error probability of any algorithm must be at least $\exp\left(-O\left(T/H_{\mathrm{Adjacent}}(\X)\right)\right)$. Our lower bound builds off of previous works \citep{kaufmann2016complexity, abbasi2018best} concerning lower bounds in best-arm identification for the multi-armed bandit setting. However, these methods alone are not sufficient for the linear setting, which we address by characterizing the lower bound as an optimization problem in Lemma \ref{thm:optlower}, and which we solve in Lemma \ref{lem:equality} by exploiting the adjacent geometry of the arm set.
    \item \textbf{Matching upper bound.} We introduce the Adjacent-optimal design, a specialization of the well-known $\X\Y$-optimal design \citep{soare2014best}, that instead reduces variance only between adjacent arms. Using the Adjacent-optimal design, we introduce the algorithm \algoname. Utilizing Lemma \ref{lem:adjacency_optimal}, we obtain Theorem \ref{thm:upper_bound1} showing that for fixed arm set $\X$, the error probability of \algoname is at most $\exp\left(-\Omega\left(T/H_{\mathrm{Adjacent}}(\X)\right)\right)$, verifying the tightness of our lower bound, and establishing $ H_{\mathrm{Adjacent}}(\X)$ as 
    the arm-set-dependent complexity measure of this setting.
\end{itemize}

\section{Related Work} \label{sec:related_work}
Best-arm identification (BAI) seeks to identify the optimal arm while ignoring cumulative reward or regret.  \citet{goos_pac_2002} marked the first substantial result in this area and the problem is now studied under two canonical settings: (i) the \emph{fixed-confidence} setting, where the learner attains a prescribed error probability target using as few pulls as possible, and (ii) the \emph{fixed-budget} setting, where the learner minimizes the error probability under a prescribed sampling budget; the latter was introduced by \citet{bubeck2009pure}.

\paragraph{Best-arm identification with fixed confidence.}
For multi-armed bandits, \citet{goos_pac_2002} established the first sample-complexity upper bound, with a nearly matching lower bound later given by \citet{mannor2004sample}. These bounds were subsequently tightened by \citet{karnin_almost_2013} and \citet{kaufmann2016complexity}. Notably, \citet{kaufmann2016complexity} derived a general information-theoretic lower bound that underpins many linear-bandit lower bounds \citep{NEURIPS2019_8ba6c657, NEURIPS2020_7212a656, camilleri2021selective}. In linear bandits, \citet{soare2014best} proposed the seminal $\mathcal{XY}$-allocation strategy, which has been refined and generalized in follow-up work \citep{tao2018best, xu2018fully, degenne_gamification_2020, NEURIPS2019_8ba6c657, NEURIPS2020_7212a656, NEURIPS2020_75800f73}. Nearly matching lower bounds were also established by \citet{NEURIPS2019_8ba6c657} and \citet{NEURIPS2020_7212a656}.

\paragraph{Best-arm identification with fixed budget.}
For multi-armed bandits, \citet{audibert2010best} provided nearly matching upper and lower bounds, with improvements due to \citet{karnin_almost_2013} and \citet{kaufmann2016complexity}. In contrast, comparatively little is known for linear bandits in this setting. Among existing results, \citet{NEURIPS2020_75800f73}—via a Gaussian-width analysis—gave the only arm-set-dependent upper bound. Methods based on G-optimal design yield arm-set-\emph{independent} guarantees \citep{ijcai2022p388, NEURIPS2022_4f9342b7}, as does the Bayesian-style approach of \citet{hoffman2014correlation}; similarly, a method based on the $\mathcal{XY}$-design in \citet{alieva_robust_2021} remains arm-set-independent. On the lower-bound side, \citet{NEURIPS2022_4f9342b7} proved a minimax bound, but to date, there is no arm-set-dependent lower bound for stationary linear bandits in the fixed-budget setting.

\paragraph{Best-arm identification in non-stationary environments.}
Compared with regret minimization, non-stationarity has received less attention in BAI. In the fixed-confidence regime, non-stationary multi-armed bandits under models different from ours were studied by \citet{jamieson2016non}, \citet{li2018hyperband}, and \citet{NEURIPS2024_e92b8441}. In the fixed-budget regime, \citet{abbasi2018best} analyzed multi-armed bandits from a best-of-both-worlds perspective and established matching upper and lower bounds in a non-stationary setting; \citet{xiong2024b} extended the upper bounds to linear bandits. However, \citet{xiong2024b} does not provide lower bounds, and its non-stationary guarantees are arm-set-independent. Our work adopts a non-stationary model similar to \citet{abbasi2018best} and \citet{xiong2024b} and, to the best of our knowledge, presents the first arm-set-dependent lower bound for the fixed-budget setting.
\section{Preliminaries} \label{sec:preliminaries}
\paragraph{Notation.} 
For $n \in \N$, denote $[n] = \{1,\dots,n\}$. For a vector $x \in \R^d$ and a positive semidefinite matrix $A \in \mathbb{S}_{+}^d$, denote $\|x\|_{A} = \sqrt{x^\top Ax}$ as the energy norm. For finite set $\X \subset \R^d$, and distribution $\lambda \in \triangle_\X$ over $\X$, denote $A(\lambda) = \sum_{x\in\X} \lambda_{x} xx^\top$. For $n \in \N$, denote $\Pi(n)$ as the set of all permutations on $[n]$.

\begin{definition}
    \label{def:adjacency}
    Given a finite set $\X\subset \R^d$, let $\mc{P}=\mathrm{conv}(\X)$ be the convex hull of $\X$, which is a polytope. We say $x\in\X$ is an \textbf{extreme point} if there exists $w \in \R^d$ such that \begin{equation}   
    \{x\} = \argmax_{y \in \mc{P}} y^\top w.\end{equation} Equivalently, the extreme points of $\X$ are the vertices of $\mc{P}$. Two distinct extreme points $x, x'\in\X$ are \textbf{adjacent} if there exists $w \in \R^d$ such that \begin{equation}  \conv(\{x,x'\}) = \argmax_{y\in \mc{P}} y^\top w.\end{equation} Equivalently, $x$ and $x'$ are adjacent if the line segment connecting them is an edge of $\mc{P}$.
\end{definition}
Let $\mc{V}_{\X}$ denote the set of extreme points of $\X$. We denote \begin{equation}
    \mc{Y}_{\X} := \left\{(x,x') \in \V_\X \times \V_\X: x\neq x'\right\}
\end{equation} as the set of all distinct extreme point pairs of $\X$. Further, we denote\begin{equation}
    \mc{I}_{\X} :=  \left\{(x,x') \in \Y_{\X} :x\text{ is adjacent to }x'\right\}
\end{equation}
 as the set of all adjacent pairs of $\X$, and, for $x\in \V_{\X}$, we denote \begin{equation}
    \mc{I}^{x}_{\X}:=\left\{x' \in \V_{\X} : (x,x') \in \I_{\X}
    \right\}
\end{equation} as the set of all points adjacent to $x$. It follows from this notation that \begin{equation}\label{eq:adjacentdef}(x,x') \in \mc{I}_{\X} \iff \exists~w \in \R^d \text{ s.t. } \{x,x'\} = \argmax_{y \in \V_{\X}} y^\top w.\end{equation}
We omit the $\X$ subscript when clear from context.

\section{BAI in Non-Stationary Linear Bandits} \label{sec:bai}
We adopt a standard model of non-stationary linear bandits with fixed horizon $T\in\mathbb{N}$. In particular, given a finite arm set $\X\subset \R^d$ with $|\X|=K$ and $\mathrm{span}(\X)=\R^d$, an adversary fixes the parameter sequence $\Bp{\theta_t}_{t=1}^T$, which remains unknown to the learner. Then, at each round $t=1, \dots, T$, the learner selects arm $x_t\in\X$ and observes a reward $r_t=x_t^\top\theta_t +\epsilon_t$, where each $\epsilon_t$ is independent, zero-mean, and 1-sub-Gaussian noise. Together, the arm set $\X$, parameter sequence $\Bp{\theta_t}_{t=1}^T$, and distribution of $\Bp{\epsilon_t}_{t=1}^T$ form an $\textbf{instance}$. The learner's objective is to identify the hindsight best arm $x_* =\argmax_{x\in\X}x^\top\otheta_T$, where $\otheta_T=\frac{1}{T}\sum_{t=1}^T\theta_t$, and for simplicity, we assume that this best arm is unique.

A central measure of difficulty in bandit problems is the reward gap between the first and second best arm. Here, we deviate slightly from the standard definition by defining this gap only in terms of extreme points $\mc{V}$. This restriction is natural since by definition, non-extreme point arms can never be the best arm, and thus do not affect problem difficulty. Formally, we define the \textbf{min-gap} $\Delta_{(1)}$ as the following:
\begin{equation}
    \Delta_{(1)} := \min_{x \in \V\setminus \{x_*\}}(x_{_*} - x)^{\top}\otheta_{T}.
\end{equation}Finally, for event $\mc{E}$, we denote $\P_{\Bp{\theta_t}_{t=1}^T}(\mc{E})$ as the probability of $\mc{E}$ under instance with parameter sequence $\Bp{\theta_t}_{t=1}^T$. Accordingly, we measure algorithmic performance by the error probability $\P_{\Bp{\theta_t}_{t=1}^T}(\widehat{x}\neq x_*)$, where $\widehat{x}$ denotes the arm output by the learner. We omit the dependence on $\Bp{\theta_t}_{t=1}^T$ when clear from context.

\subsection{Minimax-Optimal Complexity in Non-Stationary BAI}
In prior work in the non-stationary setting, \citet{xiong2024b} showed that i.i.d. sampling from the well-known G-optimal design \citep{lattimore2020bandit}
\begin{equation}\label{eq:g-optimal}
    \lambda^{G} := \argmin_{\lambda \in \triangle_{\X}} \max_{x\in\X}\|x\|^2_{A(\lambda)^{-1}}
\end{equation}
leads to a minimax-optimal error probability of \begin{equation}
    \exp\left(-\Theta\left(\frac{T} {H_{G}\left(\Delta_{(1)}\right)} \right)\right),\end{equation} where \begin{equation}
    H_{G}\left(\Delta_{(1)}\right) := \frac{d}{\Delta_{(1)}^2}.
\end{equation} The relationship between $H_{G}$ and the G-optimal design stems from the Kiefer–Wolfowitz Theorem \citep{lattimore2020bandit}, which states that \begin{equation}\label{eq:kiefer}
    \min_{\lambda \in \triangle_{\X}} \max_{x\in\X}\|x\|^2_{A(\lambda)^{-1}} = d.
\end{equation} $H_{G}$ is minimax-optimal since if the arm set is restricted to the standard basis vectors, $H_{G}$ matches the optimal complexity for the multi-armed bandit setting of $H_{\mathrm{UNIF}} = \frac{K}{\Delta_{(1)}^2}$ \citep{abbasi2018best}. However, this notion of complexity is overly pessimistic in the linear bandit setting as it ignores potential advantages arising from arm sets with more complex structure. In contrast, we refine this notion of complexity by establishing an arm-set-dependent complexity that adapts to the geometry of any given arm set.

\subsection{Adjacency and Non-Stationary BAI}
The main results of our paper are motivated by the following central lemma.
\begin{lemma}[Adjacency Lemma]\label{lem:adjacency_optimal}
    Let $\X \subset \R^d$, and $\theta\in\R^d$. For any $x \in \V$, there exists $y \in  \X $ such that $(y-x)^\top \theta > 0$ if and only if there exists $z \in \I^x$ such that $(z-x)^\top \theta > 0$.
\end{lemma}
 The formal proof of Lemma \ref{lem:adjacency_optimal} is given in Appendix \ref{appendixupper}. The idea behind the proof is quite simple: it is well known \citep{Ziegler1995} that for a polytope $P$ and a vertex $x \in P$,
\begin{equation}\label{eq:cone}
    P \subseteq x+ \mathrm{cone}\left( \left\{x'-x : x' \in \I^x\right\} \right).
\end{equation}
Thus, the quantity $(y-x)^\top \theta$ must be a conic (non-negative) combination of the points in the set $\left\{(x'-x)^\top\theta : x'\in \I^x\right\}$. Lemma \ref{lem:adjacency_optimal} follows immediately from this fact.

One implication of Lemma \ref{lem:adjacency_optimal} is that the first and second best arms of any instance must be adjacent. When constructing hard instances for our lower bound in Section \ref{sec:lower_bound}, the key strategy is to construct two instances that are hard to distinguish from each other but have different best arms. Specifically, we construct one instance with best arm $x$ and second best arm $x'$, and another instance with best arm $x'$ and second best arm $x$. Lemma \ref{lem:adjacency_optimal} tells us that to be able to do this, $x$ and $x'$ need to be adjacent. Another implication of Lemma \ref{lem:adjacency_optimal} comes from its negation, that is, if an arm is better than all its adjacent arms, then it must be the optimal arm. This implies that when identifying the best arm, we need only accurate comparisons between adjacent arms, a detail we use to design our algorithm with matching upper bound in Section \ref{sec:algorithm}. The main takeaway is that the distinguishability between adjacent arms solely determines the difficulty of identification. Given this, we introduce the following complexity measure
\begin{equation}\label{eq:hadj-def}
    H_{\mathrm{Adjacent}}\left(\X, \Delta_{(1)}\right) :=\frac{\min_{\lambda\in\triangle_{\X}}\max_{(x, x')\in\mc{I}}\Norm{x-x'}^2_{A(\lambda)^{-1}}}{\Delta_{(1)}^2}.
    \end{equation}
 By Eq. \eqref{eq:kiefer}, we observe that for any arm set $\X$ \begin{equation}
     H_{\mathrm{Adjacent}}\left(\X, \Delta_{(1)}\right) \leq 4H_{G}\left(\Delta_{(1)}\right).
\end{equation} Importantly, the above inequality can be arbitrarily loose for dense arm sets, showing that $H_{\adj}$ strictly refines the minimax-optimal complexity of \citet{xiong2024b}. For example, let $\mathcal{C}(K) \subset \R^2$ represent the arm set consisting of $K$ uniformly spaced arms on the unit circle. It is not hard to see that \begin{equation}
    \lim_{K \rightarrow \infty }\frac{H_{\mathrm{Adjacent}}\left(\mathcal{C}(K), \Delta_{(1)}\right)}{H_{G}\left(\Delta_{(1)}\right)} = \lim_{K \rightarrow \infty }\frac{\min_{\lambda\in\triangle_{\mathcal{C}(K)}}\max_{(x, x')\in\mc{I}_{\mathcal{C}(K)}}\Norm{x-x'}^2_{A(\lambda)^{-1}}}{2}= 0,
\end{equation}
since as the circle becomes more crowded, the distances between adjacent arms shrink toward zero. In the following sections, we establish $H_{\adj}$ as the arm-set-dependent complexity of this setting through a matching error probability lower and upper bound of \begin{equation}
    \exp\left(- \Theta\left(\frac{T} {H_{\mathrm{Adjacent}}\left(\X, \Delta_{(1)}\right)}\right) \right).\end{equation} 
\section{Lower Bound of BAI in Non-Stationary Linear Bandits} \label{sec:lower_bound}
We present, to the best of our knowledge, the first arm-set-dependent lower bound for BAI in the non-stationary linear bandit setting.
\begin{theorem}
    \label{thm:finallower}
    Fix $\X \subset \R^d$ and $\Delta_{(1)} > 0$. For any algorithm, there exist parameter sequences $\Bp{\theta_t}_{t=1}^{T}$ and $\Bp{\theta'_t}_{t=1}^{T}$, both with min-gap at least $\Delta_{(1)}$, such that \begin{equation}\label{eq:lower-main}
    \max\Bp{\P_{\Bp{\theta_t}_{t=1}^{T}}\Sp{\widehat{x}\neq x_*}, \P_{\Bp{\theta'_t}_{t=1}^{T}}\Sp{\widehat{x}\neq x_*}}\geq\frac{1}{4}\exp\Sp{-\frac{4T}{H_{\mathrm{Adjacent}}\left(\X, \Delta_{(1)}\right)}},
    \end{equation}
    where $ H_{\mathrm{Adjacent}}\left(\X, \Delta_{(1)}\right)$ is given by Eq. \eqref{eq:hadj-def}.
\end{theorem}

The proof is deferred to Appendix \ref{appendixlower}, of which we include a sketch in the next subsection. In Section \ref{sec:algorithm}, we present an algorithm that achieves the lower bound in Theorem \ref{thm:finallower}, verifying the tightness of our bound, and establishing $ H_{\mathrm{Adjacent}}$ as the arm-set-dependent complexity of this setting.

\subsection{Proof Sketch of Theorem \ref{thm:finallower}}
In this section, we give a proof sketch of Theorem \ref{thm:finallower}. We first obtain Lemma \ref{thm:genlower}, which characterizes the lower bound by the \emph{KL divergence} between two non-stationary multi-armed bandit instances and is a generalization of the stationary lower bound proposed in Lemma 15 of \citet{kaufmann2016complexity}. Generalizing the notation from \citet{kaufmann2016complexity}, the statement of Lemma \ref{thm:genlower} is as follows:
\begin{lemma}
    \label{thm:genlower}
    Let $\nu = \Bp{\nu_t}_{t=1}^{T}$ and $\nu' =\Bp{\nu'_t}_{t=1}^{T} $ be non-stationary bandit models with different best arms. Then
    \begin{equation}\label{eq:genlower}
    \max\left(\P_{\nu}\Sp{\widehat{k}\neq k_*}, \P_{\nu'}\Sp{\widehat{k}\neq k_*}\right) \geq \frac{1}{4} \exp \left(-\sum_{k=1}^K\sum_{t=1}^T\P_{\nu}(A_t=k) \mathrm{KL}(\nu_{t,k}, \nu'_{t,k})\right).
    \end{equation}
\end{lemma}
The proof of this lemma is deferred to Appendix \ref{appendixlower} and  very closely follows that of Lemma 15 in \citet{kaufmann2016complexity}. Applying Lemma \ref{thm:genlower} requires two instances with different best arms, which motivates the following definition.
\begin{definition}
    \label{def:feasibility}
    Given fixed arm set $\X$ and $\Delta_{(1)} > 0$, two parameter sequences $\{\theta_t\}_{t=1}^{T}$ and $\{\theta'_t\}_{t=1}^{T}$ are \textbf{feasible} if they have different best arms, and each has min-gap at least $\Delta_{(1)}$.
\end{definition}
We now proceed with the proof sketch of Theorem \ref{thm:finallower}.
\paragraph{First step.} We use Lemma \ref{thm:genlower} to obtain Lemma \ref{thm:optlower} which states the following.
\begin{lemma}[Optimization-based Lower Bound]
    \label{thm:optlower}
    Fix $\X \subset \R^d$ and $\Delta_{(1)} > 0$. For any algorithm, there exist parameter sequences $\Bp{\theta_t}_{t=1}^{T}$ and $\Bp{\theta'_t}_{t=1}^{T}$, both with min-gap at least $\Delta_{(1)}$, such that 
    \begin{equation}
        \max\Bp{\P_{\Bp{\theta_t}_{t=1}^{T}}\Sp{\widehat{x}\neq x_*}, \P_{\Bp{\theta'_t}_{t=1}^{T}}\Sp{\widehat{x}\neq x_*}} \geq \frac{1}{4} \exp \left(- T f\left(\X, \Delta_{(1)}\right) \right),
    \end{equation}
    where
    \begin{equation}\label{eq:optimization}f\left(\X, \Delta_{(1)}\right) := \begin{array}{rl}
        \max_{\lambda\in\triangle_{\X}}\min_{(x,x') \in \mc{Y}}\min_{\theta, v\in\R^d} & v^\top A(\lambda)v \\
        \mbox{\rm subject to} & (x - y)^\top\theta\geq \Delta_{(1)} \quad\forall~ y \in \V \setminus \{x\}\\
        & (x' - y)^\top(\theta + v)\geq \Delta_{(1)}\quad\forall~ y \in \V \setminus \{x'\}
    \end{array}.
    \end{equation}
\end{lemma}
 To prove Lemma \ref{thm:optlower}, the core idea behind our construction is to split the horizon into two distinct halves, similar to that of \citet{abbasi2018best}. Unlike their explicit construction, however, we characterize the hard instance through an optimization problem. At a high level, the first half allows us to minimize the KL divergence of the two instances, while the second half ensures feasibility. Concretely, first, we pick a candidate pair of distinct arms $(x,x') \in \Y$ as potential best arms for the two instances. Then, we proceed to construct the parameter sequences $\{\theta_t\}_{t=1}^{T}$ and $\{\theta'_t\}_{t=1}^{T}$. For $t \leq \frac{T}{2}$ we choose $\theta_t = 0$\footnote{The choice of $\theta_t = 0$ is arbitrary, any fixed choice such that $\theta_t'-\theta_t = v$ works.} and $\theta_t' =v$, for perturbation $v$ of our choice. Then for $t > \frac{T}{2}$ we set $\theta_t = \theta_t' = \theta$ for some $\theta$ of our choice. Further, we use the noise distribution $\epsilon_t \sim \mc{N}(0,1)$ for each $t$, resulting in a KL divergence between the reward distributions of arm $x$ at time step $t$ of \begin{equation}
     \mathrm{KL}\left(\mc{N}\left(x^\top\theta_t,1\right), \mc{N}\left(x_t^\top\theta_t',1\right)\right) =\frac{(x^\top (\theta_t - \theta'_t))^2}{2}.
 \end{equation} For such a construction, applying Lemma \ref{thm:genlower}, we obtain the lower bound\begin{equation}
     \frac{1}{4} \exp \left(-\sum_{x \in \X}\sum_{t=1}^{T/2}\P_{\{\theta_t\}_{t=1}^{T}}(x_t=x)\frac{(x^\top v)^2}{2}  \right).
 \end{equation}
 An algorithm cannot depend on the future, so for any $t \leq T/2$ we have that \begin{equation}
     \P_{\{\theta_t\}_{t=1}^{T}}(x_t=x) =\P_{\{\theta_t\}_{t=1}^{T/2}}(x_t=x).
 \end{equation}Define $\lambda \in \triangle_{\X}$ such that for each $x \in \X$
 \begin{equation}
     \lambda_x := \frac{2\sum_{t=1}^{T/2}\P_{\{\theta_t\}_{t=1}^{T/2}}(x_t = x)}{T}.
 \end{equation} Then, with some algebra, we can show that \begin{equation}
     \sum_{x \in \X}\sum_{t=1}^{T/2}\P_{\{\theta_t\}_{t=1}^{T/2}}(x_t = x) \; \frac{(x^\top v)^2}{2} =\frac{T}{4} v^\top A(\lambda)v,
 \end{equation}yielding the objective in $f\left(\X, \Delta_{(1)}\right)$. The choice of $\theta$ in the second half enforces feasibility with respect to $(x,x')$, giving rise to the constraints of $f\left(\X, \Delta_{(1)}\right)$. Crucially, since $v$, $\theta$, and $(x, x')$ are independent of $\{\theta_t\}_{t=1}^{T/2}$, we may choose them freely without affecting $\lambda$. To obtain the hardest feasible instance, we minimize over arm pairs $(x, x')$, perturbations $v$, and parameters $\theta$ to make the KL divergence as small as possible. Finally, maximizing over $\lambda$ ensures the lower bound holds for any algorithm, concluding the proof of Lemma \ref{thm:optlower}.
\paragraph{Second step.}We bound the previous optimization problem by solving the following relaxed version of the optimization problem introduced in Lemma \ref{thm:optlower}:
\begin{equation}\label{eq:optimizationrelaxed}\widehat{f}\left(\X, \Delta_{(1)}\right) := \begin{array}{rl}
        \max_{\lambda\in\triangle_{\X}}\min_{(x, x')\in\mc{I}}\min_{\theta, v\in\R^d} & v^\top A(\lambda)v \\
        \mbox{\rm subject to} & (x - y)^\top\theta\geq \Delta_{(1)} \quad\forall~y \in \V \setminus \{x\}\\
    & (x' - y)^\top(\theta + v)\geq \Delta_{(1)}\quad\forall~ y \in \V \setminus \{x'\}
    \end{array}.\end{equation}
This is a relaxation because $\widehat{f}\left(\X, \Delta_{(1)}\right)$ minimizes only over adjacent pairs $(x, x')\in\mc{I}$ rather than all extreme point pairs $(x, x')\in\mc{Y}$. Since, $\I \subseteq \Y$, and hence $\widehat{f}\left(\X, \Delta_{(1)}\right) \geq f\left(\X, \Delta_{(1)}\right)$, we can substitute $f$ by $\widehat{f}$ in the bound corresponding to Lemma \ref{thm:optlower}. Requiring $(x,x')$ to be adjacent is crucial: it enables a closed-form solution of the inner $(\theta,v)$ problem, described by the following lemma.
\begin{lemma} \label{lem:equality}
Let $\X \subset \mathbb{R}^d$ and $\Delta_{(1)} > 0$. Consider $\lambda \in \triangle_{\X}$ with $A(\lambda)$ full rank and $(x,x') \in \mc{I}$. We have:
\begin{equation}\begin{array}{rl}
    \min_{\theta, v\in\R^d} & v^\top A(\lambda)v \\
   \mbox{\rm subject to}& (x - y)^\top\theta\geq \Delta_{(1)} \quad\forall~y \in \V \setminus \{x\}\\
    & (x' - y)^\top(\theta + v)\geq\Delta_{(1)}\quad\forall~ y \in \V \setminus \{x'\}
\end{array} = \frac{4\Delta_{(1)}^2}{\|x - x'\|^2_{A(\lambda)^{-1}}}.\end{equation}
\end{lemma}
The ideas behind the proof of Lemma \ref{lem:equality} are as follows. Denote $p^* :=\frac{4\Delta_{(1)}^2}{\|x - x'\|^2_{A(\lambda)^{-1}}}.$ We first lower bound the optimization problem by $p^*$ with a simple Cauchy-Schwarz argument. Then, to upper bound the optimization problem by $p^*$, we use the key fact that for adjacent $(x, x')$, by Eq. \eqref{eq:adjacentdef}, there exists some $w \in \R^d$ such that \begin{equation}
    \{x, x'\} = \argmax_{y \in \V} y^\top w.
\end{equation} Thus, there is some $\varepsilon > 0$ such that $x^\top w = x'^\top w \geq y^\top w + \varepsilon$ for all $y \in \V\setminus \{x,x'\}$. We choose $v,u$ such that $v^\top A(\lambda)^{-1}v = p^*$ and
\begin{equation}
    (x-x')^\top u = (x'-x)^\top (u +v)= \Delta_{(1)}.
\end{equation}
Then we pick $\theta = u + \alpha w$, for some $\alpha$ to be chosen. Note that since $x^\top w = x'^\top w$, we have \begin{equation}
    (x-x')^\top\theta = (x-x')^\top u,
\end{equation}and thus \begin{equation}
    (x-x')^\top \theta = (x'-x)^\top (\theta +v)= \Delta_{(1)},
\end{equation}
regardless of the value of $\alpha$, so this constraint is always satisfied. For any $y \in \V\setminus \{x,x'\}$ we have 
\begin{equation}
    (x-y)^\top\theta \geq (x-y)^\top u + \alpha\varepsilon,
\end{equation}
and \begin{equation}
    (x'-y)^\top(\theta + v) \geq (x-y)^\top (u + v) + \alpha\varepsilon.
\end{equation}
Hence, we can increase $\alpha$ until all other constraints are satisfied, thus $(\theta,v)$ is feasible, concluding the proof of Lemma \ref{lem:equality}. Combining Lemmas \ref{thm:optlower} and \ref{lem:equality}, we obtain Theorem \ref{thm:finallower}.

\section{Matching Upper Bound}
\label{sec:algorithm}

In this section, we first introduce the Adjacent-optimal design, which is a modified version of the well-known $\mc{XY}$-optimal design. Based on this new design, we then develop the algorithm \algoname, and show that it achieves an error probability guarantee that matches the lower bound presented in Section \ref{sec:lower_bound}, verifying the tightness of our lower bound, and establishing the arm-set-dependent complexity of this setting.

\subsection{The Adjacent-Optimal Design}

A central quantity in BAI problems is the $\mathcal{XY}$-optimal design \citep{soare2014best}, which aims to uniformly minimize the variance of predictions in directions defined by the pairwise differences between all arms. Formally, the $\mathcal{XY}$-optimal design is defined as
\begin{equation}\label{eq:xy-optimal}
    \lambda^{\mc{XY}} := \argmin_{\lambda \in \triangle_{\X}} \max_{(x, x') \in \mc{Y}}\|x-x'\|^2_{A(\lambda)^{-1}}.\end{equation}
The relevance of the $\mathcal{XY}$-optimal design stems from the ranking nature of BAI; identifying the best arm only requires accurate estimates of the relative ordering between arms. Motivated by Lemma \ref{lem:adjacency_optimal}, we refine this idea by introducing the \textit{Adjacent-optimal design}, which uniformly minimizes the variance over differences only between adjacent arms. Formally, it is defined as
\begin{equation}\label{eq:adjacent-optimal}\lambda^{\mathrm{Adjacent}}:= \argmin_{\lambda \in \triangle_{\X}} \max_{(x, x') \in \mc{I}}\|x-x'\|^2_{A(\lambda)^{-1}}.\end{equation}
Lemma \ref{lem:adjacency_optimal} implies that if an arm is better than all its adjacent arms, then it must be the optimal arm. Thus, when identifying the best arm, only accurate comparisons between adjacent arms are needed. The Adjacent-optimal design embodies this principle: reducing the variance of estimation of the relative ordering between adjacent arms is sufficient for identifying the best arm. Focusing on fewer, more informative directions leads to stronger estimation, which we make use of in the algorithm presented in the next section.

\subsection{Adjacent-Optimal Best-Arm Identification (Adjacent-BAI)}
We now present the algorithm \algoname. In the first step, we compute the adjacent set $\mc{I}$, for which we provide a polynomial-time procedure in Appendix \ref{appendixadjacent}. We then compute the Adjacent-optimal design $\lambda^*$; however, we do not proceed by sampling directly from it. In order to obtain an optimal error rate using random sampling we would require an estimator that (i) admits sub-Gaussian, variance-only error and (ii) does not require a prespecified confidence level\footnote{This requirement is unnecessary in the fixed-confidence setting, hence the use of Catoni's estimator in \citet{pmlr-v139-camilleri21a}.}. According to \cite{devroye2016sub}, this is impossible without stronger distributional assumptions. Instead, we employ the classical rounding procedure of \citet{10.5555/1137748} to obtain a static allocation $\Bp{x_t}_{t=1}^{T}$ such that the empirical design matrix $\frac{1}{T}\sum_{t=1}^{T}x_tx_t^\top$ approximates the optimal design matrix $A(\lambda^*) =\sum_{x\in\X} \lambda^*_{x} xx^\top$. Specifically, the rounding procedure guarantees that if $T \geq d^2$,\footnote{We note the existence of a rounding procedure requiring only $T \geq \Omega(d)$ \citep{pmlr-v70-allen-zhu17e} at the cost of significantly looser constants. An insightful discussion on both rounding procedures is provided in Appendix B of \citet{NEURIPS2019_8ba6c657}.} the following holds:
\begin{equation}\label{eq:roundingerror}
    \max_{(x,x')\in \mc{I}}\|x - x'\|_{\left(\frac{1}{T}\sum_{t=1}^Tx_{t}x_{t}^\top\right)^{-1}}^2 \leq 2 \cdot \max_{(x,x')\in \mc{I}}\|x - x'\|_{A(\lambda^{*})^{-1}}^2.\end{equation}
 Then, to ensure our estimator is unbiased, we inject the necessary randomness by playing the allocation in a uniformly random order. Finally, we compute the least-squares estimator $\widehat{\theta}_T$ and output its corresponding best arm. The complete procedure is summarized in Algorithm \ref{algo:adjacent-bai}.\begin{algorithm}[htbp]
    \caption{Adjacent--BAI}
    \label{algo:adjacent-bai}
    \begin{algorithmic}[1]
        \STATE \textbf{Input:} budget $T\in\mathbb{N}$; arm set $\mathcal{X}\subset\R^d$
        \STATE Find $\mc{I}$, the adjacent pairs of $\mc{X}$
        \STATE $\lambda^* \leftarrow \argmin_{\lambda \in \triangle_{\X}} \max_{(x, x') \in \mc{I}}\|x-x'\|^2_{A(\lambda)^{-1}}$
        \STATE $\{x_{t}\}_{t=1}^{T} \leftarrow  \text{Round}(\lambda^*, T)$
        \STATE Sample permutation $\pi \sim \mathrm{Unif(}\Pi(T))$
        \FOR{$t=1, 2, \dots, T$}
            \STATE Play $x_{\pi(t)}$ and receive reward $r_t$
        \ENDFOR
        \STATE $\widehat{\theta}_T\leftarrow \left(\sum_{t=1}^Tx_{t}x_{t}^\top\right)^{-1}\sum_{t=1}^Tx_{\pi(t)} r_t$ 
        \STATE \textbf{return} $\widehat{x} =\argmax_{x\in\X}x^\top\widehat{\theta}_T$
    \end{algorithmic}
\end{algorithm} We characterize the error probability of Algorithm \ref{algo:adjacent-bai} in the following theorem.
\begin{theorem}[Error probability of \textsf{Adjacent-BAI}]
    \label{thm:upper_bound1}
    Fix time horizon $T \geq d^2$. Consider arm set $\mathcal{X}\subset\R^d$, and arbitrary unknown parameter sequence $\Bp{\theta_t}_{t=1}^{T}$ with min-gap $\Delta_{(1)} > 0$. Assume $\max_{x\in \X}\|x\|_2 \leq 1$ and $\max_{t \in[T]} \|\theta_t\|_2 \leq 1$. If we run Algorithm \ref{algo:adjacent-bai} in this setting and obtain $\widehat{x}$, then it holds that
    \begin{equation} \P\Sp{\widehat{x}\neq x_*}\leq \left|\I^{x_*}\right| \cdot \exp\Sp{-\frac{T}{36 \cdot H_{\mathrm{Adjacent}}\left(\X, \Delta_{(1)}\right)}},\end{equation}where $ H_{\mathrm{Adjacent}}\left(\X, \Delta_{(1)}\right)$ is given by Eq. \eqref{eq:hadj-def}.
\end{theorem}
The proof of Theorem \ref{thm:upper_bound1} is deferred to Appendix \ref{appendixupper}, of which we include a sketch in the next subsection. Notably, the upper bound matches the lower bound presented in Theorem \ref{thm:finallower} up to constants, verifying the tightness of our lower bound, and establishing $H_{\mathrm{Adjacent}}$ as the arm-set-dependent complexity of this setting.

\subsection{Proof Sketch of Theorem \ref{thm:upper_bound1}}
We outline a proof sketch of Theorem \ref{thm:upper_bound1}.
\paragraph{First step.} The statistical analysis of Theorem \ref{thm:upper_bound1} relies primarily on the following lemma.\begin{lemma}[Sub-Gaussian error of least-squares estimator]
    \label{lem:sub-Gaussian2}
    Let $\htheta_T$ be the estimator computed in Step 9 of Algorithm~\ref{algo:adjacent-bai}. For any $z \in \R^d$ we have that $z^\top(\htheta_T -\otheta_T)$ is $3\cdot\|z\|_{\left(\sum_{t=1}^Tx_tx_t^\top\right)^{-1}}$-sub-Gaussian.
\end{lemma}

To obtain Lemma \ref{lem:sub-Gaussian2}, because the arm choices $x_{\pi(t)}$ are dependent, the analysis requires particular care. Denote $A := \sum_{t=1}^Tx_tx_t^\top$. We begin by expanding the quantity of interest:
\begin{equation}
    z^\top(\htheta_T -\otheta_T) =\underbrace{z^\top\sum_{t=1}^TA^{-1}x_tx_t^\top(\theta_{\pi^{-1}(t)} - \otheta_{T})}_{:=S} + \underbrace{z^\top\sum_{t=1}^TA^{-1}x_{t}\epsilon_{\pi^{-1}(t)}}_{:=\eta},
\end{equation}
where we note that $\pi^{-1} \sim \mathrm{Unif}(\Pi(T))$. Since each $\epsilon_{\pi^{-1}(t)}$ is conditionally independent and 1-sub-Gaussian given $\pi^{-1}$, it is straightforward to show that $\eta$ is conditionally $\|z\|_{A^{-1}}$-sub-Gaussian given $\pi^{-1}$. Analyzing $S$ requires closer attention, since each $\theta_{\pi^{-1}(t)}$ is dependent, so we proceed with a martingale analysis. We construct the Doob martingale of $S$, $\{Z_t\}_{t=0}^{T}$, with $Z_t = \E[S \mid \mc{F}_t]$, where $\mc{F}_{t} = \sigma(\pi^{-1}(1), \dots, \pi^{-1}(t))$. The key observation is that, intuitively, each martingale difference $|Z_t - Z_{t-1}|$ behaves roughly on the order of $\left|z^\top A^{-1}x_tx_t^\top\theta_{\pi^{-1}(t)}\right|$. Given this, an Azuma inequality-style argument shows that \begin{equation}\E\left[\exp\left(\lambda S\right)\right] \leq \exp\left(\frac{\lambda^2}{2} \cdot  \sum_{t=1}^T8  \left\|z^\top A^{-1}x_tx_t^\top\right\|_2^2\right).
\end{equation}
Observing that 
\begin{equation}
\sum_{t=1}^{T}\left\|z^\top A^{-1}x_tx_t^\top\right\|_2^2 \leq  \|z\|_{A^{-1}}^2,
\end{equation}
we obtain that $S$ is $\sqrt{8}\|z\|_{A^{-1}}$-sub-Gaussian, completing the proof of Lemma \ref{lem:sub-Gaussian2}.
\paragraph{Second step.}
Finally, we combine Lemmas \ref{lem:adjacency_optimal} and \ref{lem:sub-Gaussian2} to obtain the final bound in Theorem \ref{thm:upper_bound1}. We aim to bound the error probability \begin{equation}
    \P\Sp{\widehat{x}\neq x_*}=\P\Sp{\exists x 
        \in \X  \text{ s.t. }x^\top\htheta_T> x_*^\top\htheta_T}.
\end{equation}
However, critically, by Lemma \ref{lem:adjacency_optimal}, we have \begin{equation}
    \P\Sp{\exists x 
        \in \X  \text{ s.t. }x^\top\htheta_T> x_*^\top\htheta_T} = \P\Sp{\exists x\in \I^{x_*}\text{ s.t. }x^\top\htheta_T>x_*^\top\htheta_T}.
\end{equation}
Applying the union bound, we have 
\begin{equation}
    \P\Sp{\exists x\in \I^{x_*}\text{ s.t. }x^\top\htheta_T >x_*^\top\htheta_T}\leq \sum_{x \in \I^{x_*}}\P\Sp{(x-x_*)^\top\left(\htheta_T -\otheta_T\right)>\Delta_{(1)}}
\end{equation}
Fix some $x \in \I^{x_*}$. Combining Lemma \ref{lem:sub-Gaussian2} and Hoeffding's inequality we have
\begin{equation}
   \P\Sp{(x-x_*)^\top\left(\htheta_T -\otheta_T\right)>\Delta_{(1)}} \leq \exp\left(- \frac{\Delta_{(1)}^2}{2\cdot3^2\cdot\| x-x_*\|_{\left(\sum_{t=1}^Tx_tx_t^\top\right)^{-1}}^2}\right).\end{equation}
Using that $(x,x_*) \in \I$, and incorporating the rounding error from Eq. \eqref{eq:roundingerror}, we have \begin{equation}
    \|x-x_*\|_{\left(\sum_{t=1}^Tx_tx_t^\top\right)^{-1}}^2 \leq \frac{2 \cdot\max_{(x, x') \in \mc{I}}\|x - x'\|_{A(\lambda^{*})^{-1}}^2}{T}.
\end{equation}
Observing that, by definition of $\lambda^*$
\begin{equation}
    \max_{(x, x') \in \mc{I}}\|x - x'\|_{A(\lambda^{*})^{-1}}^2 = \min_{\lambda \in \triangle_{\X}}\max_{(x, x') \in \mc{I}}\|x - x'\|_{A(\lambda)^{-1}}^2,
\end{equation}
completes the proof of Theorem \ref{thm:upper_bound1}.

\section{Future Work} \label{sec:futurework}
The most promising future direction is investigating whether adjacency can be leveraged to establish a stronger notion of complexity for the stationary fixed-budget setting. In the stationary fixed-budget setting, currently, no lower bound exists outside of a pessimistic minimax-optimal lower bound derived from a multi-armed bandit reduction \citep{NEURIPS2022_4f9342b7}. However, in the stationary fixed-\textit{confidence} setting, it is well-known \citep{soare2014best,NEURIPS2019_8ba6c657,NEURIPS2020_7212a656} that for arm set $\X$, and parameter vector $\theta$, the instance-optimal sample complexity is of order \begin{equation}    
\min_{\lambda \in \triangle_{\X}} \max_{x \in \X \setminus \{x_*\}} \frac{\|x_* -x\|^2_{A(\lambda)^{-1}}}{((x_* - x)^\top\theta)^2}.\end{equation}
Evaluating the above quantity, we obtain the following proposition:
\begin{proposition}\label{prop:statequiv}
     Given finite set $\X \subset\R^d$ and $\theta\in\R^d$, we have \begin{equation}
         \min_{\lambda \in \triangle_{\X}} \max_{x \in \X \setminus \{x_*\}} \frac{\|x_* -x\|^2_{A(\lambda)^{-1}}}{((x_* - x)^\top\theta)^2} = \min_{\lambda \in \triangle_{\X}} \max_{x \in \mc{I}^{x_*}} \frac{\|x_* -x\|^2_{A(\lambda)^{-1}}}{((x_* - x)^\top\theta)^2}. \end{equation}
\end{proposition}

The proof is deferred to Appendix \ref{appendixadditional}, which consists of a combination of Eq. \eqref{eq:cone} and Jensen's inequality. Proposition \ref{prop:statequiv} shows that in the stationary fixed-\textit{confidence} setting, the instance-optimal sample complexity is determined solely by the arms adjacent to $x_*$. This suggests that adjacency is fundamentally tied to the difficulty of BAI even in stationary settings. This observation, together with what we have shown in the non-stationary setting, hints towards the possibility of establishing a stronger arm-set-dependent complexity in the stationary fixed-budget setting by exploiting the adjacent geometry of the arm set.
\section*{Acknowledgements}
The work of MF was supported in part by awards NSF CCF 2212261, NSF CCF 2312775, and the Moorthy Family Professorship at UW. KJ was supported in part by a Singapore National Research Foundation AI Visiting Professorship award. Both MF and KJ were supported in part by NSF TRIPODS II DMS-2023166.

\bibliographystyle{plainnat}
\bibliography{References}

\begin{thebibliography}{31}
\providecommand{\natexlab}[1]{#1}
\providecommand{\url}[1]{\texttt{#1}}
\expandafter\ifx\csname urlstyle\endcsname\relax
  \providecommand{\doi}[1]{doi: #1}\else
  \providecommand{\doi}{doi: \begingroup \urlstyle{rm}\Url}\fi

\bibitem[Abbasi-Yadkori et~al.(2018)Abbasi-Yadkori, Bartlett, Gabillon, Malek, and Valko]{abbasi2018best}
Yasin Abbasi-Yadkori, Peter Bartlett, Victor Gabillon, Alan Malek, and Michal Valko.
\newblock Best of both worlds: Stochastic \& adversarial best-arm identification.
\newblock In \emph{Conference on learning theory}, pages 918--949. PMLR, 2018.

\bibitem[Alieva et~al.(2021)Alieva, Cutkosky, and Das]{alieva_robust_2021}
Ayya Alieva, Ashok Cutkosky, and Abhimanyu Das.
\newblock Robust {Pure} {Exploration} in {Linear} {Bandits} with {Limited} {Budget}.
\newblock In \emph{Proceedings of the 38th {International} {Conference} on {Machine} {Learning}}, pages 187--195. PMLR, July 2021.
\newblock URL \url{https://proceedings.mlr.press/v139/alieva21a.html}.
\newblock ISSN: 2640-3498.

\bibitem[Allen-Zhu et~al.(2017)Allen-Zhu, Li, Singh, and Wang]{pmlr-v70-allen-zhu17e}
Zeyuan Allen-Zhu, Yuanzhi Li, Aarti Singh, and Yining Wang.
\newblock Near-optimal design of experiments via regret minimization.
\newblock In Doina Precup and Yee~Whye Teh, editors, \emph{Proceedings of the 34th International Conference on Machine Learning}, volume~70 of \emph{Proceedings of Machine Learning Research}, pages 126--135. PMLR, 06--11 Aug 2017.
\newblock URL \url{https://proceedings.mlr.press/v70/allen-zhu17e.html}.

\bibitem[Audibert and Bubeck(2010)]{audibert2010best}
Jean-Yves Audibert and S{\'e}bastien Bubeck.
\newblock Best arm identification in multi-armed bandits.
\newblock In \emph{COLT-23th Conference on learning theory-2010}, pages 13--p, 2010.

\bibitem[Azizi et~al.(2022)Azizi, Kveton, and Ghavamzadeh]{ijcai2022p388}
MohammadJavad Azizi, Branislav Kveton, and Mohammad Ghavamzadeh.
\newblock Fixed-budget best-arm identification in structured bandits.
\newblock In Lud~De Raedt, editor, \emph{Proceedings of the Thirty-First International Joint Conference on Artificial Intelligence, {IJCAI-22}}, pages 2798--2804. International Joint Conferences on Artificial Intelligence Organization, 7 2022.
\newblock \doi{10.24963/ijcai.2022/388}.
\newblock URL \url{https://doi.org/10.24963/ijcai.2022/388}.
\newblock Main Track.

\bibitem[Bubeck et~al.(2009)Bubeck, Munos, and Stoltz]{bubeck2009pure}
S{\'e}bastien Bubeck, R{\'e}mi Munos, and Gilles Stoltz.
\newblock Pure exploration in multi-armed bandits problems.
\newblock In \emph{International conference on Algorithmic learning theory}, pages 23--37. Springer, 2009.

\bibitem[Camilleri et~al.(2021{\natexlab{a}})Camilleri, Jamieson, and Katz-Samuels]{pmlr-v139-camilleri21a}
Romain Camilleri, Kevin Jamieson, and Julian Katz-Samuels.
\newblock High-dimensional experimental design and kernel bandits.
\newblock In Marina Meila and Tong Zhang, editors, \emph{Proceedings of the 38th International Conference on Machine Learning}, volume 139 of \emph{Proceedings of Machine Learning Research}, pages 1227--1237. PMLR, 18--24 Jul 2021{\natexlab{a}}.
\newblock URL \url{https://proceedings.mlr.press/v139/camilleri21a.html}.

\bibitem[Camilleri et~al.(2021{\natexlab{b}})Camilleri, Xiong, Fazel, Jain, and Jamieson]{camilleri2021selective}
Romain Camilleri, Zhihan Xiong, Maryam Fazel, Lalit Jain, and Kevin~G Jamieson.
\newblock Selective sampling for online best-arm identification.
\newblock \emph{Advances in Neural Information Processing Systems}, 34:\penalty0 11071--11082, 2021{\natexlab{b}}.

\bibitem[Chazelle(1993)]{10.1007/BF02573985}
Bernard Chazelle.
\newblock An optimal convex hull algorithm in any fixed dimension.
\newblock \emph{Discrete Comput. Geom.}, 10\penalty0 (4):\penalty0 377–409, December 1993.
\newblock ISSN 0179-5376.
\newblock \doi{10.1007/BF02573985}.
\newblock URL \url{https://doi.org/10.1007/BF02573985}.

\bibitem[Degenne et~al.(2020)Degenne, Menard, Shang, and Valko]{degenne_gamification_2020}
Rémy Degenne, Pierre Menard, Xuedong Shang, and Michal Valko.
\newblock Gamification of {Pure} {Exploration} for {Linear} {Bandits}.
\newblock In \emph{Proceedings of the 37th {International} {Conference} on {Machine} {Learning}}, pages 2432--2442. PMLR, November 2020.
\newblock URL \url{https://proceedings.mlr.press/v119/degenne20a.html}.
\newblock ISSN: 2640-3498.

\bibitem[Devroye et~al.(2016)Devroye, Lerasle, Lugosi, and Oliveira]{devroye2016sub}
Luc Devroye, Matthieu Lerasle, Gabor Lugosi, and Roberto~I Oliveira.
\newblock Sub-gaussian mean estimators.
\newblock \emph{Annals of Statistics}, 2016.

\bibitem[Even-Dar et~al.(2002)Even-Dar, Mannor, and Mansour]{goos_pac_2002}
Eyal Even-Dar, Shie Mannor, and Yishay Mansour.
\newblock {PAC} {Bounds} for {Multi}-armed {Bandit} and {Markov} {Decision} {Processes}.
\newblock In G.~Goos, J.~Hartmanis, J.~Van~Leeuwen, Jyrki Kivinen, and Robert~H. Sloan, editors, \emph{Computational {Learning} {Theory}}, volume 2375, pages 255--270. Springer Berlin Heidelberg, Berlin, Heidelberg, 2002.
\newblock ISBN 978-3-540-43836-6 978-3-540-45435-9.
\newblock \doi{10.1007/3-540-45435-7_18}.
\newblock URL \url{http://link.springer.com/10.1007/3-540-45435-7_18}.
\newblock Series Title: Lecture Notes in Computer Science.

\bibitem[Fiez et~al.(2019)Fiez, Jain, Jamieson, and Ratliff]{NEURIPS2019_8ba6c657}
Tanner Fiez, Lalit Jain, Kevin~G Jamieson, and Lillian Ratliff.
\newblock Sequential experimental design for transductive linear bandits.
\newblock In H.~Wallach, H.~Larochelle, A.~Beygelzimer, F.~d\textquotesingle Alch\'{e}-Buc, E.~Fox, and R.~Garnett, editors, \emph{Advances in Neural Information Processing Systems}, volume~32. Curran Associates, Inc., 2019.
\newblock URL \url{https://proceedings.neurips.cc/paper_files/paper/2019/file/8ba6c657b03fc7c8dd4dff8e45defcd2-Paper.pdf}.

\bibitem[Hoffman et~al.(2014)Hoffman, Shahriari, and Freitas]{hoffman2014correlation}
Matthew Hoffman, Bobak Shahriari, and Nando Freitas.
\newblock On correlation and budget constraints in model-based bandit optimization with application to automatic machine learning.
\newblock In \emph{Artificial Intelligence and Statistics}, pages 365--374. PMLR, 2014.

\bibitem[Hou et~al.(2024)Hou, Tan, and Zhong]{NEURIPS2024_e92b8441}
Yunlong Hou, Vincent Y.~F. Tan, and Zixin Zhong.
\newblock Almost minimax optimal best arm identification in piecewise stationary linear bandits.
\newblock In A.~Globerson, L.~Mackey, D.~Belgrave, A.~Fan, U.~Paquet, J.~Tomczak, and C.~Zhang, editors, \emph{Advances in Neural Information Processing Systems}, volume~37, pages 128967--129041. Curran Associates, Inc., 2024.
\newblock \doi{10.52202/079017-4097}.
\newblock URL \url{https://proceedings.neurips.cc/paper_files/paper/2024/file/e92b8441b702d142435f243806f9ec5a-Paper-Conference.pdf}.

\bibitem[Jamieson and Talwalkar(2016)]{jamieson2016non}
Kevin Jamieson and Ameet Talwalkar.
\newblock Non-stochastic best arm identification and hyperparameter optimization.
\newblock In \emph{Artificial intelligence and statistics}, pages 240--248. PMLR, 2016.

\bibitem[Jedra and Proutiere(2020)]{NEURIPS2020_7212a656}
Yassir Jedra and Alexandre Proutiere.
\newblock Optimal best-arm identification in linear bandits.
\newblock In H.~Larochelle, M.~Ranzato, R.~Hadsell, M.F. Balcan, and H.~Lin, editors, \emph{Advances in Neural Information Processing Systems}, volume~33, pages 10007--10017. Curran Associates, Inc., 2020.
\newblock URL \url{https://proceedings.neurips.cc/paper_files/paper/2020/file/7212a6567c8a6c513f33b858d868ff80-Paper.pdf}.

\bibitem[Karmarkar(1984)]{10.1145/800057.808695}
N.~Karmarkar.
\newblock A new polynomial-time algorithm for linear programming.
\newblock In \emph{Proceedings of the Sixteenth Annual ACM Symposium on Theory of Computing}, STOC '84, page 302–311, New York, NY, USA, 1984. Association for Computing Machinery.
\newblock ISBN 0897911334.
\newblock \doi{10.1145/800057.808695}.
\newblock URL \url{https://doi.org/10.1145/800057.808695}.

\bibitem[Karnin et~al.(2013)Karnin, Koren, and Somekh]{karnin_almost_2013}
Zohar Karnin, Tomer Koren, and Oren Somekh.
\newblock Almost {Optimal} {Exploration} in {Multi}-{Armed} {Bandits}.
\newblock In \emph{Proceedings of the 30th {International} {Conference} on {Machine} {Learning}}, pages 1238--1246. PMLR, May 2013.
\newblock URL \url{https://proceedings.mlr.press/v28/karnin13.html}.
\newblock ISSN: 1938-7228.

\bibitem[Katz-Samuels et~al.(2020)Katz-Samuels, Jain, karnin, and Jamieson]{NEURIPS2020_75800f73}
Julian Katz-Samuels, Lalit Jain, zohar karnin, and Kevin~G Jamieson.
\newblock An empirical process approach to the union bound: Practical algorithms for combinatorial and linear bandits.
\newblock In H.~Larochelle, M.~Ranzato, R.~Hadsell, M.F. Balcan, and H.~Lin, editors, \emph{Advances in Neural Information Processing Systems}, volume~33, pages 10371--10382. Curran Associates, Inc., 2020.
\newblock URL \url{https://proceedings.neurips.cc/paper_files/paper/2020/file/75800f73fa80f935216b8cfbedf77bfa-Paper.pdf}.

\bibitem[Kaufmann et~al.(2016)Kaufmann, Capp{\'e}, and Garivier]{kaufmann2016complexity}
Emilie Kaufmann, Olivier Capp{\'e}, and Aur{\'e}lien Garivier.
\newblock On the complexity of best-arm identification in multi-armed bandit models.
\newblock \emph{The Journal of Machine Learning Research}, 17\penalty0 (1):\penalty0 1--42, 2016.

\bibitem[Lattimore and Szepesv{\'a}ri(2020)]{lattimore2020bandit}
Tor Lattimore and Csaba Szepesv{\'a}ri.
\newblock \emph{Bandit Algorithms}.
\newblock Cambridge University Press, 2020.
\newblock \doi{10.1017/9781108571401}.

\bibitem[Li et~al.(2018)Li, Jamieson, DeSalvo, Rostamizadeh, and Talwalkar]{li2018hyperband}
Lisha Li, Kevin Jamieson, Giulia DeSalvo, Afshin Rostamizadeh, and Ameet Talwalkar.
\newblock Hyperband: A novel bandit-based approach to hyperparameter optimization.
\newblock \emph{Journal of Machine Learning Research}, 18\penalty0 (185):\penalty0 1--52, 2018.

\bibitem[Mannor and Tsitsiklis(2004)]{mannor2004sample}
Shie Mannor and John~N Tsitsiklis.
\newblock The sample complexity of exploration in the multi-armed bandit problem.
\newblock \emph{Journal of Machine Learning Research}, 5\penalty0 (Jun):\penalty0 623--648, 2004.

\bibitem[Pukelsheim(2006)]{10.5555/1137748}
Friedrich Pukelsheim.
\newblock \emph{Optimal Design of Experiments (Classics in Applied Mathematics) (Classics in Applied Mathematics, 50)}.
\newblock Society for Industrial and Applied Mathematics, USA, 2006.
\newblock ISBN 0898716047.

\bibitem[Soare et~al.(2014)Soare, Lazaric, and Munos]{soare2014best}
Marta Soare, Alessandro Lazaric, and R{\'e}mi Munos.
\newblock Best-arm identification in linear bandits.
\newblock \emph{Advances in neural information processing systems}, 27, 2014.

\bibitem[Tao et~al.(2018)Tao, Blanco, and Zhou]{tao2018best}
Chao Tao, Sa{\'u}l Blanco, and Yuan Zhou.
\newblock Best arm identification in linear bandits with linear dimension dependency.
\newblock In \emph{International Conference on Machine Learning}, pages 4877--4886. PMLR, 2018.

\bibitem[Xiong et~al.(2024)Xiong, Camilleri, Fazel, Jain, and Jamieson]{xiong2024b}
Zhihan Xiong, Romain Camilleri, Maryam Fazel, Lalit Jain, and Kevin Jamieson.
\newblock A/b testing and best-arm identification for linear bandits with robustness to non-stationarity.
\newblock In \emph{International Conference on Artificial Intelligence and Statistics}, pages 1585--1593. PMLR, 2024.

\bibitem[Xu et~al.(2018)Xu, Honda, and Sugiyama]{xu2018fully}
Liyuan Xu, Junya Honda, and Masashi Sugiyama.
\newblock A fully adaptive algorithm for pure exploration in linear bandits.
\newblock In \emph{International Conference on Artificial Intelligence and Statistics}, pages 843--851. PMLR, 2018.

\bibitem[Yang and Tan(2022)]{NEURIPS2022_4f9342b7}
Junwen Yang and Vincent Tan.
\newblock Minimax optimal fixed-budget best arm identification in linear bandits.
\newblock In S.~Koyejo, S.~Mohamed, A.~Agarwal, D.~Belgrave, K.~Cho, and A.~Oh, editors, \emph{Advances in Neural Information Processing Systems}, volume~35, pages 12253--12266. Curran Associates, Inc., 2022.
\newblock URL \url{https://proceedings.neurips.cc/paper_files/paper/2022/file/4f9342b74c3bb63f6e030d8263082ab6-Paper-Conference.pdf}.

\bibitem[Ziegler(1995)]{Ziegler1995}
G{\"u}nter~M. Ziegler.
\newblock \emph{Lectures on Polytopes}, volume 152 of \emph{Graduate Texts in Mathematics}.
\newblock Springer-Verlag, New York, 1995.
\newblock ISBN 978-0-387-97495-8.
\newblock \doi{10.1007/978-1-4613-8431-1}.

\end{thebibliography}

\newpage
\appendix

\section{Lower Bound on Error Probability}\label{appendixlower}

\subsection{Proof of Theorem \ref{thm:finallower}}
\begin{proof}
    Fix an algorithm. By Lemma \ref{thm:optlower} we have there exists parameter sequences $\Bp{\theta_t}_{t=1}^T$ and $\Bp{\theta_t'}_{t=1}^T$, both with min-gap at least $\Delta_{(1)}$ such that \begin{equation}\max\Bp{\P_{\Bp{\theta_t}_{t=1}^{T}}\Sp{\widehat{x}\neq x_*}, \P_{\Bp{\theta'_t}_{t=1}^{T}}\Sp{\widehat{x}\neq x_*}} \geq \frac{1}{4} \exp \left(- T f\left(\X, \Delta_{(1)}\right) \right),\end{equation}
where $f\left(\X, \Delta_{(1)}\right)$ is defined in Lemma \ref{thm:optlower}. Denote \begin{equation}\hat{f}(\V, \Delta_{(1)}) = \begin{array}{rl}
        \max_{\lambda\in\triangle_{\X}}\min_{(x, x')\in\mc{I}}\min_{\theta, v\in\R^d} & v^\top A(\lambda)v \\
        \mbox{\rm subject to} & (x - y)^\top\theta\geq \Delta_{(1)} \quad\forall~y \in \V \setminus \{x\}\\
    & (x' - y)^\top(\theta + v)\geq \Delta_{(1)}\quad\forall~ y \in \V \setminus \{x'\}
    \end{array}.\end{equation}
    It follows that $f\left(\X, \Delta_{(1)}\right) \leq \hat{f}(\V, \Delta_{(1)})$. By Lemma $\ref{lem:equality}$, it follows that \begin{equation}\hat{f}(\V, \Delta_{(1)}) = \max_{\lambda\in\triangle_{\X}}\min_{(x, x')\in\mc{I}} \frac{4\Delta_{(1)}^2}{\|x - x'\|^2_{A(\lambda)^{-1}}}.\end{equation} Thus, we have that for any algorithm, there exist parameter sequences $\Bp{\theta_t}_{t=1}^{T}$ and $\Bp{\theta'_t}_{t=1}^{T}$, both with min-gap at least $\Delta_{(1)}$, such that
\begin{align*}
    \max\Bp{\P_{\Bp{\theta_t}_{t=1}^{T}}\Sp{\widehat{x}\neq x_*}, \P_{\Bp{\theta'_t}_{t=1}^{T}}\Sp{\widehat{x}\neq x_*}} &\geq \frac{1}{4} \exp \left(- T f\left(\X, \Delta_{(1)}\right) \right) \\
    &\geq \frac{1}{4} \exp \left(- T \hat{f}(\V, \Delta_{(1)}) \right) \\
    &=  \frac{1}{4} \exp \left(- \max_{\lambda\in\triangle_{\X}}\min_{(x, x')\in\mc{I}} \frac{4T\Delta_{(1)}^2}{\|x - x'\|^2_{A(\lambda)^{-1}}}\right)\\
    &= \frac{1}{4} \exp \left(-  \frac{4T\Delta_{(1)}^2}{\min_{\lambda\in\triangle_{\X}}\max_{(x, x')\in\mc{I}}\|x - x'\|^2_{A(\lambda)^{-1}}}\right).
\end{align*}
\end{proof}

\subsection{Technical Lemmas}
\subsubsection{Proof of Lemma \ref{thm:genlower}}
Before proceeding, we generalize the notation introduced in \citet{kaufmann2016complexity}. Fix an algorithm. For $T \in \N$, let $\nu = \{\nu_t\}_{t=1}^T$, be a non-stationary bandit with $K \in \N$ arms, where $\nu_{t,k}$ is the distribution of arm $k \in [K]$ at timestep $t \in [T]$.  We denote $A_t$ and $Z_t$ as the action and reward, respectively, of the algorithm at timestep $t$ under $\nu$.  We denote $\mc{F}_t = \sigma(A_1, Z_1,\dots, A_t, Z_t)$ as the information available to the algorithm after time step $t$. We denote $\P_\nu$ as the law of $(Z_t)_{t=1}^T$, and $\E_\nu$ as the expectation under $\P_\nu$. We denote $\widehat{k}$ represent the arm returned by the algorithm and $k_*= \argmax_k \sum_{t=1}^{T}\nu_{t,k}$ as the best arm. 

For non-stationary bandit ${\nu}$, and for any $k \in [K], t\in [T]$ there exists measure $\lambda_{t,k}$ such that $\nu_{t,k}$ has density $f_{t,k}$ \citep{kaufmann2016complexity}. Under non-stationary bandits  $\nu$ and $\nu'$, with each $\nu_{t,k}$ and $\nu'_{t,k}$ mutually absolutely continuous, denote the log-likelihood of $(Z_s)_{s=1}^t$ as\begin{equation}L_t=\sum_{k=1}^K\sum_{s=1}^t\mathds{1}_{\Bp{A_s=k}}\log\Sp{\frac{f_{s, k}(Z_s)}{f_{s, k}'(Z_s)}},\end{equation} where $A_s$ is the arm played by the algorithm at timestep $s$.

The following is a restatement of Lemma 18 from \citet{kaufmann2016complexity} generalized to non-stationary log-likelihood $L_\sigma$: 
\begin{lemma}\label{lem:stoppingtime}
    Let $\sigma$ be a stopping time with respect to $\mc{F}_t$. For every event $\mc{E} \in \mc{F}_\sigma$,
    \begin{equation}\P_{\nu'}(\mathcal{E}) = \E_{\nu}[\mathds{1}_{\mathcal{E}}\exp(-
L_{\sigma})].\end{equation}
\end{lemma}

The proof follows the exact same structure as Lemma 18 from \citet{kaufmann2016complexity}, with a slight modification to the inductive step.
\begin{proof}
    For any $k \in [K]$, let $(Y_{t,k})_{t \in \N}$ be a sequence such that if $A_t = k$ then $Z_t  =Y_{t,k}$. First, we aim to show that for any $n \in \N$ and measurable function $g:\R^n \rightarrow \R$, \begin{equation}\E_{\nu'}[g(Z_1,\dots,Z_{n})] = \E_{\nu}[g(Z_1,\dots,Z_{n})\exp\left(-L_{n}(Z_1, \ldots, Z_n)\right)].\end{equation} When $n =1 $, $\E_{\nu'}[g(Z_1)] = \E_{\nu}[g(Z_1)\exp\left(-L_{1})\right)]$ from the proof of \citet{kaufmann2016complexity} Lemma 18. 
    
    Now assume the statement holds for some $n \in \N$. We aim to show it holds for $n + 1$. Let  $g:\R^{n+1} \rightarrow \R$ be a measurable function. We have \begin{align*}
    &\E_{\nu'}[g(Z_1,\dots,Z_{n + 1})]\\
    =& \E_{\nu}\left[\sum_{k=1}^{K} \mathds{1}_{A_{n+1}=k} \E_{\nu'} \left[ g(Z_1, \ldots, Z_n, Y_{n+1,k}) \mid \mathcal{F}_n \right] \exp\left(-L_n(Z_1, \ldots, Z_n)\right)
\right] \\
=& \E_{\nu}\left[\sum_{k=1}^{K} \mathds{1}_{A_{n+1}=k} \int  g(Z_1, \ldots, Z_n, z) \frac{f_{n+1, k}'(z)}{f_{n+1, k}(z)}f_{n+1, k}(z)d\lambda_{n+1,k}\exp\left(-L_n(Z_1, \ldots, Z_n)\right)
\right]. \\
\end{align*}
On the event $(A_{n +1} = k)$, we have $L_{n+1}(Z_1, \ldots, Z_n,z) = L_{n}(Z_1, \ldots, Z_n) + \log\frac{f_{n+1, k}(z)}{f'_{n+1, k}(z)}$, thus:
\begin{align*}
    &\E_{\nu'}[g(Z_1,\dots,Z_{n + 1})] \\
    =& \E_{\nu}\left[\sum_{k=1}^{K} \mathds{1}_{A_{n+1}=k} \int  g(Z_1, \ldots, Z_n, z) \exp\left(-L_{n+1}(Z_1, \ldots, Z_n, z)\right)f_{n+1, k}(z)d\lambda_{n+1,k} \right]\\
    =& \E_{\nu}[g(Z_1,\dots,Z_{n + 1})\exp\left(-L_{n+1}(Z_1, \ldots, Z_n, Z_{n+1})\right)].
\end{align*}
It follows that this holds for all $n \in \N$, and thus for every $\mc{E} \in \mc{F}_n$ we have
\begin{equation}\P_{\nu'}(\mathcal{E}) = \E_{\nu}[\mathds{1}_{\mathcal{E}}\exp(-
L_{n})].\end{equation}
Let $\sigma$ be a stopping time with respect to $\mc{F}_t$. From the proof of \citet{kaufmann2016complexity} Lemma 18, it follows that for every event $\mc{E} \in \mc{F}_\sigma$:
    \begin{equation}\P_{\nu'}(\mathcal{E}) = \E_{\nu}[\mathds{1}_{\mathcal{E}}\exp(-
L_{\sigma})].\end{equation}
\end{proof}

We can now prove Lemma \ref{thm:genlower}.
\begin{proof}
    Denoting $\rho_\nu, \rho_{\nu'}$ the distributions of $\widehat{k}$ under $\nu$ and $\nu'$, respectively, from the proof of Lemma 15 from \citet{kaufmann2016complexity} we have
    \begin{equation}\max\left(\P_{\nu}\Sp{\widehat{k}\neq k_*}, \P_{\nu'}\Sp{\widehat{k}\neq k_*}\right) \geq \frac{1}{4}\exp\left(-\mathrm{KL}\left(\rho_\nu, \rho_{\nu'}\right)\right).\end{equation} 
    From Lemma \ref{lem:stoppingtime} and the proof of Lemma 19 from \citet{kaufmann2016complexity} it follows that for any event $\mathcal{E}$:
    \begin{equation}\P_{\nu'}(\mathcal{E}) \geq \exp\left(-\E_{\nu}\left[L_{\sigma} \mid \mathcal{E}\right]\right)\P_{\nu}(\mathcal{E}),\end{equation}
    and thus,
    \begin{equation}\E_{\nu}\left[L_{\sigma} \mid \mathcal{E}\right] \geq \log \frac{\P_{\nu}(\mathcal{E})}{\P_{\nu'}(\mathcal{E})}.\end{equation}
    It follows that $\E_{\nu}\left[L_{T} \mid \widehat{k} \right] \geq \log \frac{\P_{\nu}(\widehat{k})}{\P_{\nu'}(\widehat{k})}$. So, using the tower rule we obtain:
    \begin{align*}
        \E_{\nu}[L_{T}] &= \E_{\nu}[\E_{\nu}[L_{T}\mid \widehat{k}]] \\
        &\geq \E_{\nu}\left[\log \frac{\P_{\nu}(\widehat{k})}{\P_{\nu'}(\widehat{k})}\right] \\
        &= \mathrm{KL}\left(\rho_\nu, \rho_{\nu'}\right).
    \end{align*}
Now we have \begin{equation}\max(\P_{\nu}\Sp{\widehat{k}\neq k_*}, \P_{\nu'}\Sp{\widehat{k}\neq k_*}) \geq \frac{1}{4}\exp\left(-\E_{\nu}[L_{T}]\right).\end{equation}
    It follows that:
    \begin{align*}
        &\max(\P_{\nu}\Sp{\widehat{k}\neq k_*}, \P_{\nu'}\Sp{\widehat{k}\neq k_*}) \\&\geq 
        \frac{1}{4} \exp \left(-\E_{\nu}[L_T]\right)\\
        &=\frac{1}{4} \exp \left(-\E_{\nu}\left[\sum_{k=1}^K\sum_{t=1}^T\mathds{1}_{\Bp{A_t=k}}\log\Sp{\frac{f_{k, t}(Z_t)}{f_{k, t}'(Z_t)}} \right]\right) \\
        &= \frac{1}{4} \exp \left(-\sum_{k=1}^K\sum_{t=1}^T\E_{\nu}\left[\mathds{1}_{\Bp{A_t=k}}\log\Sp{\frac{f_{k, t}(Z_t)}{f_{k, t}'(Z_t)}} \right]\right) \\
        &= \frac{1}{4} \exp \left(-\sum_{k=1}^K\sum_{t=1}^T\E_{\nu}\left[\E_{\nu}\left[\mathds{1}_{\Bp{A_t=k}}\log\Sp{\frac{f_{k, t}(Z_t)}{f_{k, t}'(Z_t)}} \Bigg| A_t = k\right]\right]\right) \\
        &= \frac{1}{4} \exp \left(-\sum_{k=1}^K\sum_{t=1}^T\E_{\nu}\left[\mathds{1}_{\Bp{A_t=k}}\E_{\nu}\left[\log\Sp{\frac{f_{k, t}(Z_t)}{f_{k, t}'(Z_t)}} \Bigg| A_t = k\right]\right]\right) \\
        &= \frac{1}{4} \exp \left(-\sum_{k=1}^K\sum_{t=1}^T\E_{\nu}\left[\mathds{1}_{\Bp{A_t=k}}\mathrm{KL}(\nu_{t,k}, \nu'_{t,k})\right]\right) \\
        &= \frac{1}{4} \exp \left(-\sum_{k=1}^K\sum_{t=1}^T\E_{\nu}\left[\mathds{1}_{\Bp{A_t=k}}\right]\mathrm{KL}(\nu_{t,k}, \nu'_{t,k})\right) \\
        &= \frac{1}{4} \exp \left(-\sum_{k=1}^K\sum_{t=1}^T\P_{\nu}(A_t=k) \cdot \mathrm{KL}(\nu_{t,k}, \nu'_{t,k})\right)
    \end{align*}
\end{proof}
\subsubsection{Proof of Lemma \ref{thm:optlower}}

\begin{proof}
    Without loss of generality, assume $T$ is even. For $t \leq \frac{T}{2}$, set $\theta_t = 0$ and $\theta'_t = v$, for some $v$ to be chosen later. For $t>\frac{T}{2}$ set $\theta_t =\theta_t' = \theta$, for $\theta$ to be chosen later. Assume $\Bp{\theta_t}_{t=1}^{T}$ and $\Bp{\theta'_t}_{t=1}^{T}$ have min-gap at least $\Delta_{(1)}$ and have different best arms. Further, we let $\epsilon_t \sim \mc{N}(0,1)$ for each $t$. Fix an algorithm. Denote \begin{equation}p(x_t = x) := \P_{\{\theta_t\}_{t=1}^{T/2}}(x_t=x).\end{equation} By Lemma \ref{thm:genlower} we have
    \begin{align*}&\max\Bp{\P_{\Bp{\theta_t}_{t=1}^{T}}\Sp{\widehat{x}\neq x_*}, \P_{\Bp{\theta'_t}_{t=1}^{T}}\Sp{\widehat{x}\neq x_*}}\\ 
    &\geq \frac{1}{4} \exp \left(-\sum_{x \in \X}\left(\sum_{t=1}^{T/2}\P_{\{\theta_t\}_{t=1}^{T}}(x_t=x)\frac{(x^\top v)^2}{2}  + \sum_{t=T/2 + 1}^{T}\P_{\{\theta_t\}_{t=1}^{T}}(x_t=x)  \frac{(x^\top \theta - x^\top \theta)^2}{2}\right)\right) \\
         &= \frac{1}{4} \exp \left(- \sum_{x \in \X} \sum_{t=1}^{T/2}\P_{\{\theta_t\}_{t=1}^{T}}(x_t=x)\frac{(x^\top v)^2}{2}\right) \\
         &= \frac{1}{4} \exp \left(- \sum_{x \in \X} \sum_{t=1}^{T/2}\P_{\{\theta_t\}_{t=1}^{T/2}}(x_t=x)\frac{(x^\top v)^2}{2}\right) \\
        &= \frac{1}{4} \exp \left(- \sum_{x \in \X}  \sum_{t=1}^{T/2}p(x_t = x)\frac{(x^\top v)^2}{2}\right),
    \end{align*}
where the second-to-last equality follows because algorithms cannot depend on the future. Notice that \begin{align*}
\sum_{x \in \X}  \sum_{t=1}^{T/2}p(x_t = x)\frac{(x^\top v)^2}{2}
    &=\frac{1}{2}v^\top\left(\sum_{x \in \X}  \sum_{t=1}^{T/2}p(x_t = x)xx^\top\right) v \\
    &= \frac{T}{4}v^\top\left(\sum_{x \in \X}\sum_{t=1}^{T/2}  \frac{2p(x_t = x)}{T}xx^\top\right) v.
\end{align*}
Denote $\lambda_x:=\sum_{t=1}^{T/2}\frac{2 p(x_t = x) }{T}$. It follows that $\lambda \in \triangle_{\X}$. We thus have \begin{equation}\sum_{x \in \X}\sum_{t=1}^{T/2}  p(x_t = x)\frac{(x^\top v)^2}{2} = \frac{T}{4}v^\top A(\lambda)v.\end{equation} We now have there exist parameter sequences $\Bp{\theta_t}_{t=1}^{T}$ and $\Bp{\theta'_t}_{t=1}^{T}$ such that: 
\begin{equation}\max\Bp{\P_{\Bp{\theta_t}_{t=1}^{T}}\Sp{\widehat{x}\neq x_*}, \P_{\Bp{\theta'_t}_{t=1}^{T}}\Sp{\widehat{x}\neq x_*}} \geq \frac{1}{4} \exp \left(- \frac{T}{4}v^\top A(\lambda)v \right),\end{equation}
subject to $\Bp{\theta_t}_{t=1}^{T}$ and $\Bp{\theta'_t}_{t=1}^{T}$ having different best arms and maintaining min-gap $\Delta_{(1)}$. The previous constraints can be written formally as ``$\exists (x,x') \in \mc{Y}$ such that $\frac{1}{T}\sum_{t=1}^{T}(x - y)^\top\theta_t \geq \Delta_{(1)}$ for all $y \in \V \setminus \{x\}$, and $ \frac{1}{T}\sum_{t=1}^{T}(x' - y)^\top\theta_t \geq \Delta_{(1)}$ for all $y \in \V \setminus \{x'\}$". By our construction this is equivalent to ``$\exists (x,x') \in \mc{Y}$ such that $(x-y)^\top\theta' \geq \Delta_{(1)}$ for all $y \in \V \setminus \{x\}$, and $ (x'-y)^\top(\theta' + v') \geq \Delta_{(1)}$ for all $y \in \V \setminus \{x'\}$", where $\theta' = \frac{\theta}{2}$ and $v' = \frac{v}{2}$. Fix some $(x,x') \in \mc{Y}$. Since $\Bp{\theta_t}_{t=1}^{T/2}$ is independent of $\theta'$ and $v'$, we have $\lambda$ is also independent of $\theta'$ and $v'$, so we are free to choose any value of $\theta',v'$. Thus, we may choose the $\theta',v'$ such that the right-hand side above is maximized. Formally, we now have for some $(x,x') \in \mc{Y}$, there exist parameter sequences $\Bp{\theta_t}_{t=1}^{T}$ and $\Bp{\theta'_t}_{t=1}^{T}$ such that: 
\begin{equation}\max\Bp{\P_{\Bp{\theta_t}_{t=1}^{T}}\Sp{\widehat{x}\neq x_*}, \P_{\Bp{\theta'_t}_{t=1}^{T}}\Sp{\widehat{x}\neq x_*}} \geq \max_{\theta',v' \in \R^d}\frac{1}{4} \exp \left(- Tv'^\top A(\lambda)v' \right),\end{equation}
subject to $(x - y)^\top\theta' \geq \Delta_{(1)}$ for all $y \in \V \setminus \{x\}$, and $ (x' - y)^\top(\theta' + v') \geq \Delta_{(1)}$ for all $y \in \V \setminus \{x'\}$. Similarly, since $\lambda$ is independent of $(x,x')$, we may choose the $(x,x')$ pair that maximizes the above right-hand side. That is, we now have there exist parameter sequences $\Bp{\theta_t}_{t=1}^{T}$ and $\Bp{\theta'_t}_{t=1}^{T}$ such that: 
\begin{equation}\max\Bp{\P_{\Bp{\theta_t}_{t=1}^{T}}\Sp{\widehat{x}\neq x_*}, \P_{\Bp{\theta'_t}_{t=1}^{T}}\Sp{\widehat{x}\neq x_*}} \geq \max_{(x,x') \in \mc{Y}}\max_{\theta',v' \in \R^d}\frac{1}{4} \exp \left(-  Tv'^\top A(\lambda)v' \right),\end{equation}
subject to $(x - y)^\top\theta' \geq \Delta_{(1)}$ for all $y \in \V \setminus \{x\}$, and $ (x' - y)^\top(\theta' + v') \geq \Delta_{(1)}$ for all $y \in \V \setminus \{x'\}$. Now for this bound to hold for any algorithm, we take a minimum over every possible $\lambda$ that could be induced by an algorithm, obtaining that for any algorithm there exist parameter sequences $\Bp{\theta_t}_{t=1}^{T}$ and $\Bp{\theta'_t}_{t=1}^{T}$  such that: 
\begin{align*}
    \max\Bp{\P_{\Bp{\theta_t}_{t=1}^{T}}\Sp{\widehat{x}\neq x_*}, \P_{\Bp{\theta'_t}_{t=1}^{T}}\Sp{\widehat{x}\neq x_*}} &\geq \min_{\lambda \in \triangle_{\X}}\max_{(x,x') \in \mc{Y}}\max_{\theta',v' \in \R^d}\frac{1}{4} \exp \left(-  Tv'^\top A(\lambda)v' \right) \\
    &= \frac{1}{4} \exp \left(- T \max_{\lambda \in \triangle_{\X}}\min_{(x,x') \in \mc{Y}}\min_{\theta',v' \in \R^d} v'^\top A(\lambda)v'\right),
\end{align*}
subject to $(x - y)^\top\theta' \geq \Delta_{(1)}$ for all $y \in \V \setminus \{x\}$, and $ (x' - y)^\top(\theta' + v') \geq \Delta_{(1)}$ for all $y \in \V \setminus \{x'\}$ which is exactly \begin{equation}\max\Bp{\P_{\Bp{\theta_t}_{t=1}^{T}}\Sp{\widehat{x}\neq x_*}, \P_{\Bp{\theta'_t}_{t=1}^{T}}\Sp{\widehat{x}\neq x_*}} \geq \frac{1}{4} \exp \left(- T f\left(\X, \Delta_{(1)}\right) \right),\end{equation}
by definition of $f\left(\X, \Delta_{(1)}\right)$.
\end{proof}

\subsubsection{Proof of Lemma \ref{lem:equality}}

\begin{proof}
    Direct corollary of 
    Lemma \ref{lem:optlower} and \ref{lem:optupper}.
\end{proof}

\begin{lemma} \label{lem:optlower}
Let $\X \subset \mathbb{R}^d$ and $\Delta_{(1)} > 0$. Consider positive definite $A \in \R^{d\times d}$ and $(x,x') \in \mc{Y}$. We have:
    \begin{equation}\begin{array}{rl}
    \min_{\theta, v\in\R^d} & v^\top Av \\
    \mbox{\rm subject to} & (x - y)^\top\theta\geq \Delta_{(1)} \quad\forall~y \in \V \setminus \{x\}\\
    & (x' - y)^\top(\theta + v)\geq \Delta_{(1)}\quad\forall~ y \in \V \setminus \{x'\}
\end{array} \geq \frac{4\Delta_{(1)}^2}{\|x - x'\|^2_{A^{-1}}}.\end{equation}
\end{lemma}

\begin{proof}
    From the constraints we have that $(x - x')^\top\theta\geq \Delta_{(1)}$ and $(x' - x)^\top(\theta + v)\geq \Delta_{(1)}$. Adding these constraints together we obtain that feasible $v$ must satisfy $(x' - x)^\top v \geq 2\Delta_{(1)}$. Consider feasible $v$. We have: \begin{align*}
    2\Delta_{(1)} &\leq  (x' - x)^\top v \\
    &\leq \|x'-x\|_{A^{-1}} \cdot \|v\|_{A} && \text{(Cauchy-Schwarz)}\\
    &= \|x-x'\|_{A^{-1}} \cdot \sqrt{v^{\top}Av}.
\end{align*}
Squaring both sides, we obtain \begin{equation}4\Delta_{(1)}^2 \leq \|x-x'\|_{A^{-1}}^2 \cdot v^{\top}Av.\end{equation}
That is 
\begin{equation}\frac{4\Delta_{(1)}^2}{\|x-x'\|_{A^{-1}}^2} \leq v^{\top}Av.\end{equation}
This holds for any feasible $v$, thus it must hold at the optimum.
\end{proof}

\begin{lemma} \label{lem:optupper}
    Let $\X \subset \mathbb{R}^d$ and $\Delta_{(1)} > 0$. Consider positive definite $A \in \R^{d\times d}$ and $(x,x') \in \mc{I}$. Suppose $(x,x') \in \mc{I}$. We have:
\begin{equation}\begin{array}{rl}
    \min_{\theta, v\in\R^d} & v^\top Av \\
    \mbox{\rm subject to} & (x - y)^\top\theta\geq \Delta_{(1)} \quad\forall~y \in \V \setminus \{x\}\\
    & (x' - y)^\top(\theta + v)\geq \Delta_{(1)}\quad\forall~ y \in \V \setminus \{x'\}
\end{array} \leq \frac{4\Delta_{(1)}^2}{\|x - x'\|^2_{A^{-1}}}.\end{equation}
\end{lemma}

\begin{proof}
    By definition of $\mc{I}$ there exists $w \in \R^d$ such that \begin{equation}\{x,x'\} = \argmax_{y \in \mc{V}} y^\top w.\end{equation} So, there exists $\varepsilon > 0$, such that $x^\top w =x'^\top w \geq y^\top w + \varepsilon $ for all $y \in \V \setminus \{x,x'\}$. Choose: \begin{equation}\theta =\frac{\Delta_{(1)}}{\|x-x'\|_{A^{-1}}^2}A^{-1}(x-x') + \alpha w,\end{equation} for some $\alpha$ to be chosen later, and \begin{equation}v = - \frac{2\Delta_{(1)}}{\|x-x'\|_{A^{-1}}^2}A^{-1}(x-x').\end{equation} First, it is clear that: \begin{equation}v^\top Av = \frac{4\Delta_{(1)}^2}{\|x - x'\|^2_{A^{-1}}}.\end{equation} We need only check feasibility now. For the rest of the proof denote: \begin{equation}B_{z,y} := \frac{\Delta_{(1)}}{\|x-x'\|_{A^{-1}}^2}(z- y)^\top A^{-1}(x-x'),\end{equation} for $z \in \{x,x'\}$ and some $y$. Let us first analyze the first block of constraints. Consider the case when $y = x'$. Since $(x -x')^\top w = 0$, we have: \begin{align*}
        (x - x')^\top\theta &= \frac{\Delta_{(1)}}{\|x-x'\|_{A^{-1}}^2}(x-x')^{\top}A^{-1}(x-x')+ \alpha(x -x')^\top w\\
        &= \Delta_{(1)}
    \end{align*}
    which satisfies the constraint. Consider $y \in \V \setminus \{x,x'\}$. Using that $(x -y)^\top w \geq \varepsilon$, we have: \begin{align*}
        (x - y)^\top\theta &= \frac{\Delta_{(1)}}{\|x-x'\|_{A^{-1}}^2}(x - y)^\top A^{-1}(x-x') + \alpha(x -y)^\top w\\
        &= B_{x,y} + \alpha(x -y)^\top w \\
        &\geq B_{x,y} + \alpha \varepsilon,
    \end{align*}
    which we want to be greater than $\Delta_{(1)}$ to fulfill the constraint. This tells us $\alpha $ must satisfy: \begin{equation}\alpha \geq \frac{\Delta_{(1)}- \min_{y}B_{x,y}}{\varepsilon}.\end{equation}
    Now let's analyze the second block of constraints. Consider $y=x$. We have:
    \begin{align*}
        (x' - x)^\top(\theta + v) &= \frac{\Delta_{(1)} - 2\Delta_{(1)}}{\|x-x'\|_{A^{-1}}^2}(x' - x)^\top A^{-1}(x-x') + \alpha(x' -x)^\top w\\
        &= -(-\Delta_{(1)})\\
        &= \Delta_{(1)},
    \end{align*}
    which satisfies the constraint. Consider $y \in \V \setminus \{x,x'\}$. We have \begin{align*}
        (x' - y)^\top\theta &= \frac{\Delta_{(1)} - 2\Delta_{(1)}}{\|x-x'\|_{A^{-1}}^2}(x' - y)^\top A^{-1}(x-x') + \alpha(x' -y)^\top w\\
        &=\frac{-\Delta_{(1)}}{\|x-x'\|_{A^{-1}}^2}(x' - y)^\top A^{-1}(x-x') + \alpha(x' -y)^\top w \\
        &= - B_{x',y} + \alpha(x' -y)^\top w \\
        &= - B_{x',y} + \alpha(x -y)^\top w \\
        &\geq - B_{x',y}  + \alpha \varepsilon,
    \end{align*}
    which we also want to be greater than $\Delta_{(1)}$. This means $\alpha$ must also satisfy: \begin{equation}\alpha \geq \frac{\Delta_{(1)}+ \max_{y}B_{x',y}}{\varepsilon}.\end{equation}
Choosing:
    \begin{equation}\alpha = \frac{\Delta_{(1)} + \max\{-\min_{y}B_{x,y},\max_{y}B_{x',y}\}}{\varepsilon},\end{equation}
    completes the proof.
\end{proof}

\section{Error probability of Algorithm \ref{algo:adjacent-bai}}\label{appendixupper}
\subsection{Proof of Theorem \ref{thm:upper_bound1}}
\begin{proof}
    We have
    \begin{align*}
        \P\Sp{\widehat{x}\neq x_*}&=\P\Sp{\exists x 
        \in \X  \text{ s.t. }x^\top\htheta_T>x_*^\top\htheta_T}\\
        &= \P\Sp{\exists x\in \I^{x_*}\text{ s.t. }x^\top\htheta_T> x_*^\top\htheta_T}  \tag{Lemma \ref{lem:adjacency_optimal}}\\
        &\leq  \sum_{x \in \I^{x_*} }\P\Sp{\htheta_T^\top(x-x_*)> 0}\\
        &= \sum_{x \in \I^{x_*}}\P\Sp{\htheta_T^\top(x-x_*) - \otheta_T^\top(x-x_*)>-\otheta_T^\top(x-x_*)}\\
        &\leq \sum_{x \in \I^{x_*}}\P\Sp{\htheta_T^\top(x-x_*) - \otheta_T^\top(x-x_*)> \Delta_{(1)}}. \\
        &=  \sum_{x \in \I^{x_*}}\P\Sp{(x-x_*)^\top\left(\htheta_T -\otheta_T\right)> \Delta_{(1)}} \\
        &\leq \sum_{x \in \I^{x_*}}\exp\left(- \frac{\Delta_{(1)}^2}{2\cdot 3^2 \cdot \|x - x_*\|_{\left(\sum_{t=1}^Tx_tx_t^\top\right)^{-1}}^2}\right) \tag{Lemma \ref{lem:sub-Gaussian2} and Hoeffding's inequality}\\
        &= \sum_{x \in \I^{x_*}}\exp\left(- \frac{\Delta_{(1)}^2}{18 \cdot\|x-x_*\|_{\left(\sum_{t=1}^Tx_{\pi(t)}x_{\pi(t)}^\top\right)^{-1}}^2}\right) \\
        &\leq \sum_{x \in \I^{x_*}}\exp\left(- \frac{\Delta_{(1)}^2}{18 \cdot \max_{(x, x') \in \mc{I}}\|x - x'\|_{\left(\sum_{t=1}^Tx_{\pi(t)}x_{\pi(t)}^\top\right)^{-1}}^2}\right) \\
        &\leq \sum_{x \in \I^{x_*}}\exp\left(- \frac{T\Delta_{(1)}^2}{18 \cdot2\max_{(x, x') \in \mc{I}}\|x - x'\|_{A(\lambda^{*})^{-1}}^2}\right) \tag{Rounding error of $\lambda^*$}\\
        &= \sum_{x \in \I^{x_*}}\exp\left(- \frac{T\Delta_{(1)}^2}{36 \cdot \min_{\lambda \in \triangle_{\X}} \max_{(x, x') \in \mc{I}}\|x - x'\|_{A(\lambda)^{-1}}^2}\right) \\
        &= \left|\I^{x_*}\right| \cdot \exp\left(- \frac{T\Delta_{(1)}^2}{36 \cdot \min_{\lambda \in \triangle_{\X}} \max_{(x, x') \in \mc{I}}\|x - x'\|_{A(\lambda)^{-1}}^2}\right).
    \end{align*}
\end{proof}
\subsection{Technical Lemmas for Theorem \ref{thm:upper_bound1}}

\subsubsection{Proof of Lemma \ref{lem:adjacency_optimal}}
\begin{proof}
    The backward direction is trivial. We now prove the forward direction. Consider $x \in \X$, $\theta \in \R^d$, such that there exists $y \in \X $ with $(y-x)^\top \theta > 0$. By Lemma 3.6 of \cite{Ziegler1995}, $\conv(\X) \subset x + \mathrm{cone}\left(\{z -x: z \in \I^x\}\right)$. Thus $y \in x + \mathrm{cone}\left(\{z -x: z \in \I^x\}\right)$, that is $y - x \in\mathrm{cone}(\{z -x: z \in \I^x\})$. Equivalently $\exists \alpha_z \geq 0$ for each $z \in \I^x$, such that $y-x = \sum_{z \in \I^x} \alpha_z (z-x)$. Suppose for contradiction that for every $z \in \I^x$, $(z-x)^\top\theta \leq  0$. We have: 
    \begin{align*}
        (y-x)^{\top}\theta &= \sum_{z \in \I^x} \alpha_z (z-x)^{\top}\theta \\
        &\leq  0. 
    \end{align*}
    Which is a contradiction. Thus there exists $z \in \I^x$ such that $(z-x)^\top \theta >0$. 
\end{proof}
\subsubsection{Proof of Lemma \ref{lem:sub-Gaussian2}}
\begin{proof}
 Denote $A := \sum_{t=1}^Tx_tx_t^\top$. Expanding $z^\top\left(\widehat{\theta}_T -\otheta_{T}\right)$ we have
    \begin{align*}
    z^\top\left(\widehat{\theta}_T -\otheta_{T}\right) &= z^\top\left(A^{-1} \sum_{t=1}^Tx_{\pi(t)}r_t - \otheta_{T}\right)\\
    &= z^\top\left(A^{-1}\sum_{t=1}^Tx_{\pi(t)}x_{\pi(t)}^\top\theta_t - \otheta_{T} + A^{-1}\sum_{t=1}^Tx_{\pi(t)}\epsilon_t\right) \\
    &= z^\top\left(A^{-1}\sum_{t=1}^Tx_tx_t^\top\theta_{\pi^{-1}(t)} - \otheta_{T} + A^{-1}\sum_{t=1}^Tx_{t}\epsilon_{\pi^{-1}(t)}\right)\\
        &= z^\top\left(A^{-1}\sum_{t=1}^Tx_tx_t^\top\theta_{\pi^{-1}(t)} - A^{-1}\sum_{t=1}^Tx_tx_t^\top\otheta_{T} + A^{-1}\sum_{t=1}^Tx_{t}\epsilon_{\pi^{-1}(t)}\right) \\
        &= z^\top\left(A^{-1}\sum_{t=1}^Tx_tx_t^\top(\theta_{\pi^{-1}(t)} - \otheta_{T}) + A^{-1}\sum_{t=1}^Tx_{t}\epsilon_{\pi^{-1}(t)}\right) \\
        &= z^\top\left(\sum_{t=1}^TA^{-1}x_tx_t^\top(\theta_{\pi^{-1}(t)} - \otheta_{T}) + \sum_{t=1}^TA^{-1}x_{t}\epsilon_{\pi^{-1}(t)}\right) \\
        &= z^\top\sum_{t=1}^TA^{-1}x_tx_t^\top(\theta_{\pi^{-1}(t)} - \otheta_{T}) + z^\top\sum_{t=1}^TA^{-1}x_{t}\epsilon_{\pi^{-1}(t)}.
    \end{align*}
    Denote \begin{equation}S = z^\top\sum_{t=1}^TA^{-1}x_tx_t^\top(\theta_{\pi^{-1}(t)} - \otheta_{T}),\end{equation}
    and \begin{equation}\eta = z^\top\sum_{t=1}^TA^{-1}x_{t}\epsilon_{\pi^{-1}(t)}.\end{equation}  Note that $\pi^{-1} \sim \mathrm{Unif}(\Pi(T))$, thus, by Lemma \ref{lem:sub-Gaussian1} we have $S$ is $\sqrt{8}\|z\|_{A^{-1}}$-sub-Gaussian. Observe that each $\epsilon_{\pi^{-1}(t)}$ is conditionally independent and 1-sub-Gaussian given $\pi^{-1}$. Analyzing the conditional MGF of $\eta$ for some $\lambda$, we have:
\begin{align*}
    \mathbb{E}\left[\exp\left(\lambda z^\top\sum_{t=1}^TA^{-1}x_{t}\epsilon_{\pi^{-1}(t)}\right) \, \Bigg | \, \pi^{-1} \right]
    &=  \prod_{t=1}^T\mathbb{E}\left[\exp\left(\lambda z^\top A^{-1}x_{t}\epsilon_{\pi^{-1}(t)}\right) \mid \pi^{-1} \right]  \tag{Cond. Ind. given  $\pi^{-1}$} \\
    &\leq \prod_{t=1}^T\exp\left(\frac{\lambda^2 z^\top A^{-1}x_{t}x_{t}^\top A^{-1} z}{2}\right)  \tag{Cond. 1-sub-G. given $\pi^{-1}$} \\
    &= \exp\left(\frac{\lambda^2 z^\top A^{-1}\sum_{t=1}^Tx_{t}x_{t}^\top A^{-1} z}{2}\right)\\
    &= \exp\left(\frac{\lambda^2 \|z\|_{A^{-1}}^2}{2}\right).    
\end{align*}
So $\eta$ is conditionally $\|z\|_{A^{-1}}$-sub-Gaussian given $\pi^{-1}$. Analyzing the MGF of $z^\top\left(\widehat{\theta}_T -\otheta_{T}\right)$, for some $\lambda$ we have \begin{align*}
    \E \left[\exp\left(\lambda z^\top\left(\widehat{\theta}_T -\otheta_{T}\right)\right)\right] &= \E \left[\exp\left(\lambda S + \lambda \eta\right)\right] \\
    &= \E \left[\exp\left(\lambda S\right)\exp\left(\lambda\eta\right)\right]\\
    &= \E \left[\exp\left(\lambda S\right)\E \left[\exp\left(\lambda\eta\right) \mid \pi^{-1}\right]\right] \\
    &\leq \E \left[\exp\left(\lambda S\right)\right] \cdot \exp\left(\frac{\lambda^2 \|z\|_{A^{-1}}^2}{2}\right)\\
    &\leq  \exp\left(\frac{\lambda^28\|z\|_{A^{-1}}^2}{2}\right)\cdot \exp\left(\frac{\lambda^2 \|z\|_{A^{-1}}^2}{2}\right) \\
    &= \exp\left(\frac{\lambda^2(1+8)\|z\|_{A^{-1}}^2}{2}\right).
    \end{align*}
    Thus $z^\top\left(\widehat{\theta}_T -\otheta_{T}\right)$is $\sqrt{9}\|z\|_{A^{-1}}$-sub-Gaussian.
\end{proof}

\begin{lemma}\label{lem:sub-Gaussian1}
   Consider fixed sequences $\{x_t\}_{t=1}^T$ spanning $\R^d$ and $\{\theta_t\}_{t=1}^T$, with $\max_t\|x_t\| \leq B$ and $\max_t\|\theta_t\| \leq M$. Denote $\otheta_{T} = \frac{1}{T}\sum_{t=1}^T \theta_t$, $A = \sum_{t=1}^{T} x_{t}x_{t}^\top$, and $\tilde{\theta}_t = \theta_{t} - \otheta_{T}$ for each $t \in [T]$. Let $\pi \sim \mathrm{Unif}(\Pi(T))$. For any $z \in \R^d$ we have that $z^\top\sum_{t=1}^TA^{-1}x_tx_t^\top\tilde{\theta}_{\pi(t)}$ is $\sqrt{8}MB\|z\|_{A^{-1}}$-sub-Gaussian.
\end{lemma}
\begin{proof}
   Fix $z \in \R^d$, and denote $S = z^\top\sum_{t=1}^TA^{-1}x_tx_t^\top\tilde{\theta}_{\pi(t)} $. For $t \in [T]$ define $Z_t = \E[S \mid \mc{F}_t]$, with $\mc{F}_t = \sigma(\pi(1),\dots,\pi(t))$. Further define $Z_0 = \E[S]$. By construction, we have that $\{Z_t\}_{t=0}^T$ is a martingale and $Z_T = S$. Also note that $Z_0 = \E[S] = 0$, since $\sum_{t=1}^{T}\tilde{\theta}_t = 0$. Denote $\Delta_t = Z_t - Z_{t-1}$. We begin by bounding $|\Delta_t|$ for some $t< T$. To ease notation, let us denote \begin{equation}a_t = x_tx_t^\top A^{-1}z,\end{equation} and  \begin{equation}\mu_t = \frac{1}{T-t+1}\sum_{k \in [T] \setminus \{\pi(1),\dots,\pi(t-1)\}} \tilde{\theta}_k.\end{equation} Let $t < T$. We have:
   \begin{align*}
       Z_{t-1} &= \sum_{s=1}^{t - 1} a_s^\top\tilde{\theta}_{\pi(s)}   + \sum_{s = t }^T a_s^\top\left(\frac{1}{T - t+1}\sum_{k \in [T] \setminus \{\pi(1),\dots,\pi(t-1)\}} \tilde{\theta}_k\right) \\
       &= \sum_{s=1}^{t - 1} a_s^\top\tilde{\theta}_{\pi(s)}   + \sum_{s = t }^T a_s^\top\mu_t \\
       &= \sum_{s=1}^{t - 1} a_s^\top\tilde{\theta}_{\pi(s)}   + a_t^\top\mu_t + \sum_{s = t +1}^T a_s^\top\mu_t,
   \end{align*}
    and
    \begin{align*}
        Z_{t} &= \sum_{s=1}^{t} a_s^\top\tilde{\theta}_{\pi(s)}   + \sum_{s = t +1}^T a_s^\top\left(\frac{1}{T - t}\sum_{k \in [T] \setminus \{\pi(1),\dots,\pi(t)\}} \tilde{\theta}_k\right)\\
        &=\sum_{s=1}^{t - 1} a_s^\top\tilde{\theta}_{\pi(s)}  + a_t^\top\tilde{\theta}_{\pi(t)} + \sum_{s = t +1}^T a_s^\top\left(\frac{1}{T - t}\left(\sum_{k \in [T] \setminus \{\pi(1),\dots,\pi(t-1)\}} \tilde{\theta}_k - \tilde{\theta}_{\pi(t)}\right)\right) \\
    &= \sum_{s=1}^{t - 1} a_s^\top\tilde{\theta}_{\pi(s)}  + a_t^\top\tilde{\theta}_{\pi(t)} + \sum_{s = t +1}^T a_s^\top\left(\frac{(T-t+1)\mu_t - \tilde{\theta}_{\pi(t)}}{T-t}\right).\\
    \end{align*}
It follows that:
\begin{align*}
    \Delta_{t} &= Z_t - Z_{t-1} \\
    &= a_t^\top\tilde{\theta}_{\pi(t)} - a_t^\top\mu_t+ \sum_{s = t +1}^T a_s^\top  \left(\frac{(T-t+1)\mu_t - \tilde{\theta}_{\pi(t)}}{T-t}  - \mu_t\right)\\
    &= a_t^\top\left(\tilde{\theta}_{\pi(t)} - \mu_t\right) + \sum_{s = t +1}^T a_s^\top   \left(\frac{(T-t+1)\mu_t - \tilde{\theta}_{\pi(t)}}{T-t}  - \mu_t\right)\\
    &= a_t^\top\left(\tilde{\theta}_{\pi(t)} -\mu_t\right) + \frac{1}{T-t}\sum_{s = t +1}^T a_s^\top  \left((T-t+1)\mu_t - \tilde{\theta}_{\pi(t)} - (T-t)\mu_t\right)\\
    &= a_t^\top\left(\tilde{\theta}_{\pi(t)} - \mu_t\right) + \frac{1}{T-t}\sum_{s = t +1}^T a_s^\top \left(\mu_t -\tilde{\theta}_{\pi(t)}\right) \\
    &= \left(a_t - \frac{1}{T-t}\sum_{s = t +1}^T a_s\right)^\top\left(\tilde{\theta}_{\pi(t)} - \mu_t\right).
\end{align*}
We have \begin{align*}
    |\Delta_{t}| &= \left|\left(a_t - \frac{1}{T-t}\sum_{s = t +1}^T a_s\right)^\top\left(\tilde{\theta}_{\pi(t)} - \mu_t\right)\right| \\
    &\leq \left\|a_t - \frac{1}{T-t}\sum_{s = t +1}^T a_s\right\|_2 \cdot \left\|\tilde{\theta}_{\pi(t)} - \mu_t\right\|_2 \\
    &= \left\|a_t - \frac{1}{T-t}\sum_{s = t +1}^T a_s\right\|_2 \cdot \left\|\tilde{\theta}_{\pi(t)} - \frac{1}{T-t+1}\sum_{k \in [T] \setminus \{\pi(1),\dots,\pi(t-1)\}} \tilde{\theta}_k\right\|_2\\
    &= \left\|a_t - \frac{1}{T-t}\sum_{s = t +1}^T a_s\right\|_2 \cdot \left\|\theta_{\pi(t)} -\otheta_{T} - \frac{1}{T-t+1}\sum_{k \in [T] \setminus \{\pi(1),\dots,\pi(t-1)\}} \theta_k +\otheta_{T}\right\|_2\\
    &= \left\|a_t - \frac{1}{T-t}\sum_{s = t +1}^T a_s\right\|_2 \cdot \left\|\theta_{\pi(t)}  - \frac{1}{T-t+1}\sum_{k \in [T] \setminus \{\pi(1),\dots,\pi(t-1)\}} \theta_k \right\|_2\\
    &\leq \left\|a_t - \frac{1}{T-t}\sum_{s = t +1}^T a_s\right\|_2 \cdot \left(\left\|\theta_{\pi(t)}\right\|_2  + \frac{1}{T-t+1}\sum_{k \in [T] \setminus \{\pi(1),\dots,\pi(t-1)\}} \left\|\theta_k \right\|_2\right)\\
    &\leq \left\|a_t - \frac{1}{T-t}\sum_{s = t +1}^T a_s\right\|_2 \cdot \left(M  + \frac{1}{T-t+1}\sum_{k \in [T] \setminus \{\pi(1),\dots,\pi(t-1)\}} M\right) \\
    &= \left\|a_t - \frac{1}{T-t}\sum_{s = t +1}^T a_s\right\|_2 \cdot 2M \\
    &:= c_t.
\end{align*}
Bounding $|\Delta_T|$ we have \begin{align*}
    |\Delta_T| &= |Z_T - Z_{T-1}| \\
    &= \left|a_{T}^{\top}\tilde{\theta}_{\pi(T)} - a_{T}^{\top}\tilde{\theta}_{\pi(T)}\right| \\
    &= 0 \\
    &:= c_T.
\end{align*}
Note that by construction, $\E[\Delta_t \mid \mc{F}_{t-1}] = 0$ for any $t$. For some $\lambda$, we have:
    \begin{align*}
        \E\left[\exp\left(\lambda S\right)\right] &= \E\left[\exp\left(\lambda (Z_T - Z_0)\right)\right] \\
        &= \E\left[\exp\left(\lambda \sum_{t=1}^T\Delta_t\right)\right] \\
        &= \E\left[\exp\left(\lambda \sum_{t=1}^{T-1}\Delta_t\right)\cdot\exp\left(\lambda \Delta_T\right)\right] \\
        &= \E\left[\exp\left(\lambda \sum_{t=1}^{T-1}\Delta_t\right)\cdot\E\left[\exp\left(\lambda \Delta_T\right) \mid \mc{F}_{T-1}\right]\right] && (\text{Tower rule}) \\
        &\leq \E\left[\exp\left(\lambda \sum_{t=1}^{T-1}\Delta_t\right)\right] \cdot\exp\left(\frac{\lambda^2}{2} \cdot  c_T^2\right) && (\text{Hoeffding's lemma}) \\
        &\leq \cdots \\
        &\leq \exp\left(\frac{\lambda^2}{2} \cdot  \sum_{t=1}^T c_{t}^2\right). && 
        (\text{Induction})
    \end{align*}
It now only remains to bound: \begin{equation}\sum_{t=1}^{T} c_t^2 = 4M^2 \cdot \sum_{t=1}^{T-1}\left\|a_t - \frac{1}{T-t}\sum_{s = t +1}^T a_s\right\|_2^2.\end{equation}For each $t \in [T-1]$ denote $d_t = a_t - \frac{1}{T-t}\sum_{s = t +1}^T a_s$, and construct the matrices
\begin{equation}\mathbf{a} : = \begin{bmatrix}
    a_1 & \cdots &a_{T}
\end{bmatrix} \in \R^{d \times T}\end{equation}
\begin{equation}\mathbf{d} : = \begin{bmatrix}
    d_1 & \cdots & d_{T-1}
\end{bmatrix} \in \R^{d \times  T-1}.\end{equation}Further consider the matrix $D \in \R^{T \times T-1}$ \begin{equation}D_{i,j} := \begin{cases}
    1 & i = j \\
    -\frac{1}{T - j} & i > j \\
    0 &i < j
\end{cases}.\end{equation}
It follows by construction that $\mathbf{a} D = \mathbf{d}$. Let $D_i$ denote the $i$th column of $D$. Note that for any $i \in [T -1]$ we have
\begin{align*}
    \|D_i\|_2^2 &= 1^2 + (T- i)\left(\frac{1}{T-i}\right)^2 \\
    &= 1 + \frac{1}{T-i} \\
    &\leq 2.
\end{align*}
Consider $D_i$, $D_j$, for $i < j$. Note that \begin{align*}
    D_i^\top D_j &= \left( \frac{-1}{T-i}\right) \cdot 1 + \sum_{k =j +1}^{T}\left( \frac{-1}{T-i}\right)\left( \frac{-1}{T-j}\right) \\
    &= \left( \frac{-1}{T-i}\right) + (T-j)\left( \frac{-1}{T-i}\right)\left( \frac{-1}{T-j}\right) \\
    &= 0.
\end{align*}
Thus, each column of $D$ is orthogonal. This means \begin{equation}D^\top D = \text{diag}\left(\|D_1\|_2^2 ,\dots ,\|D_{T-1}\|_2^2\right)\end{equation}
and so,
\begin{align*}
    \|D\|_{\text{op}}^2 &= \lambda_{\max}\left(D^\top D\right) \\
    &= \max_{i} \|D_i\|_2^2 \\
    &\leq 2.
\end{align*}
It follows that
\begin{align*}
    \sum_{t=1}^{T} c_t^2 &=4M^2 \cdot \sum_{t=1}^{T-1}\left\|a_t - \frac{1}{T-t}\sum_{s = t +1}^T a_s\right\|_2^2 \\
    & =4M^2 \cdot \| \mathbf{d}\|_F^2  \\
    &= 4M^2 \cdot\| \mathbf{a}D\|_F^2 \\ 
    &\leq 4M^2 \cdot\| \mathbf{a}\|_F^2\|D\|_{\text{op}}^2 \\
    &\leq 2 \cdot4M^2 \cdot   \| \mathbf{a}\|_F^2\\
    &= 8M^2 \sum_{t=1}^{T}\left\|a_t\right\|_2^2 \\
    &= 8M^2 \sum_{t=1}^{T} z^\top A^{-1}x_tx_t^\top x_tx_t^\top A^{-1}z \\
    &\leq 8M^2B^2 \sum_{t=1}^{T} z^\top A^{-1}x_tx_t^\top A^{-1}z \\
    &= 8M^2B^2  z^\top A^{-1}\sum_{t=1}^{T}x_tx_t^\top  A^{-1}z \\
    &= 8M^2B^2 \|z\|_{A^{-1}}^2.
\end{align*}
It follows that, for all $\lambda$ we have: 
\begin{align*}
        \E\left[\exp\left(\lambda S\right)\right] &\leq \exp\left(\frac{\lambda^2}{2} \cdot  \sum_{t=1}^T c_{t}^2\right) \\
        &\leq \exp\left(\frac{\lambda^2}{2} \cdot  8M^2B^2\|z\|_{A^{-1}}^2\right).
    \end{align*}
Thus $S$ is $\sqrt{8}MB\|z\|_{A^{-1}}$-sub-Gaussian.
\end{proof}

\section{Computing The Adjacent Set $\mc{I}$}\label{appendixadjacent}

Consider $\mathcal{X} \subset \R^d$ with $|\X| = K$. If one has access to the convex hull of $\mathcal{X}$, the adjacent set $\mathcal{I}$ can be extracted in $O(|\mc{I}|) =O(K^2)$ time. If one does not have access to the convex hull, one can compute it in $O\left(K \log K +K^{\lfloor d/2 \rfloor}\right)$ time \citep{10.1007/BF02573985}, which is efficient for small $d$. When $d$ is large, one can use the following procedure to compute $\mathcal{I}$ in time polynomial in $K$ and $d$. First center $\X$, that is, denoting $\bar{x} = \frac{1}{K}\sum_{x \in \X} x$, compute \begin{equation}
    \X' \gets \left\{x - \bar{x}:x \in \X\right\}.
\end{equation}
This shift ensures $0 \in \mathrm{int}(\conv(\X'))$, which will be important later on. For each $x \in \mathcal{X}'$ solve:
$$\begin{array}{rl}
        \max_{\epsilon,w} & \epsilon \\
        \text{subject to} & x^\top w = 1\\
        & y^\top w\leq  1 - \epsilon\quad\forall~ y \in \mathcal{X}' \setminus \{x\}
    \end{array}.\quad (\text{LP}1)$$
Then, check if the resulting $\epsilon > 0$. If so, add $x$ to the set $\mathcal{V}_{\X'}$. It follows that $\mathcal{V}_{\X'}$ is exactly the set of all extreme points of $\mathcal{X}'$. Now, for each linearly independent pair $x,x' \in \mathcal{V}_{\X'}$ solve:
$$\begin{array}{rl}
        \max_{\epsilon,w} & \epsilon \\
        \text{subject to} & x^\top w = 1\\
        & x'^\top w = 1\\
        & y^\top w\leq  1 - \epsilon\quad\forall~ y \in \mathcal{V}_{\X'} \setminus \{x,x'\}
    \end{array}. \quad (\text{LP}2)$$
Then, check if the resulting $\epsilon > 0$. If so, add $(x,x')$ to the set $\mathcal{I}_{\X'}$. It follows that $\mathcal{I}_{\X'}$ is exactly the set of all adjacent pairs of $\mathcal{X}'$. Since the structure of a polytope is preserved under translation, we have \begin{equation}
    \V_{\X} = \left\{x + \bar{x}:x \in \mathcal{V}_{\X'}\right\},
\end{equation} and \begin{equation}\I_{\X} = \left\{(x + \bar{x}, x' + \bar{x}):(x,x') \in \mathcal{I}_{\X'}\right\}. 
\end{equation} The correctness of LP1 is the following. Consider some $x \in \X'$. Since $0 \in \mathrm{int}(\conv(\X'))$, any hyperplane strictly supporting $x$ must be of the form $\left\{y \in \R^d : y^\top h = c\right\}$, for $c > 0$ and normal vector $h$. Thus $\left\{y \in \R^d : y^\top \frac{h}{c} = 1\right\}$ also strictly supports $x$. Hence, to find whether or not $x$ has a strict supporting hyperplane, it suffices to restrict our attention to normal vectors in the set $W(x) = \left\{w:x^\top w =1\right\}$. If $\epsilon > 0$, $w$ corresponds exactly to a strict supporting hyperplane of $x$, and the definition of extreme point is satisfied. If $\epsilon \leq 0$, then no $w \in W(x)$ is a strict supporting hyperplane of $x$, and the definition cannot be satisfied. The correctness of LP2 follows analogously, with one extra caveat: checking only linearly independent pairs is sufficient; if $x,x'$ are linearly dependent then for any $w$ we have $x^\top w = x'^\top w\Rightarrow x^\top w =0$, hence no $w$ strictly supports both $x$ and $x'$ since $0 \in \mathrm{int}(\conv(\X'))$. Finally, both LP1 and LP2 have $d+1$ variables and at most $K$ constraints, so each can be solved in $\text{poly}(K,d)$ time \citep{10.1145/800057.808695}. Thus, the extreme point procedure takes in total $K \cdot \text{poly}(K,d)$, and the adjacent procedure takes in total $K^2 \cdot \text{poly}(K,d)$. Thus, the total run time remains $\text{poly}(K,d)$.

\section{Proof of Proposition \ref{prop:statequiv}}\label{appendixadditional}
\begin{proof}
    Note that since $\mc{I}^{x_*} \subseteq \X \setminus \{x_*\}$ we have $$\min_{\lambda \in \triangle_{\X}} \max_{x \in \X \setminus \{x_*\}} \frac{\|x_* -x\|^2_{A(\lambda)^{-1}}}{((x_* - x)^\top\theta)^2} \geq \min_{\lambda \in \triangle_{\X}} \max_{x \in \mc{I}^{x_*}} \frac{\|x_* -x\|^2_{A(\lambda)^{-1}}}{((x_* - x)^\top\theta)^2}. $$
    By Lemma \ref{lem:equivalence} we have 
    $$\min_{\lambda \in \triangle_{\X}} \max_{x \in \X \setminus \{x_*\}} \frac{\|x_* -x\|^2_{A(\lambda)^{-1}}}{((x_* - x)^\top\theta)^2} \leq \min_{\lambda \in \triangle_{\X}} \max_{x \in \mc{I}^{x_*}} \frac{\|x_* -x\|^2_{A(\lambda)^{-1}}}{((x_* - x)^\top\theta)^2}. $$
\end{proof}
\begin{lemma}\label{lem:equivalence}
    Consider $\theta \in \mathbb{R}^d$ and $M \succ 0$, and let $x_* = \argmax_{x \in \X}x^\top\theta$. Consider $x \in \X \setminus \{x_*\}$. We have $$\frac{\|x_* - x\|_{M}^{2}}{((x_* - x)^{\top}\theta)^{2}}  \leq\max_{x' \in \mc{I}^{x_*}}\frac{\|x_* - x'\|_{M}^{2}}{((x_* - x')^{\top}\theta)^{2}}.$$
\end{lemma}
\begin{proof}
     By Lemma 3.6 of \cite{Ziegler1995} we have $\X \subset x_* + \text{cone}\left(\{y -x_*: y \in \mc{I}^{x_*}\}\right)$. Thus $\exists~\alpha_y \geq 0$ for each $y \in \mc{I}^{x_*}$, such that $x_*-x = \sum_{y \in \mc{I}^{x_*}} \alpha_y (x_*-y)$. Thus we have 
    \begin{align*}
        \left|\left|\frac{x_* - x}{(x_* - x)^{\top}\theta}\right|\right|_{M}^2 &= \left|\left|\frac{\sum_{y \in \mc{I}^{x_*}} \alpha_y (x_*-y)}{\sum_{x \in \mc{I}^{x_*}} \alpha_x (x_*-x)^{\top}\theta}\right|\right|_{M}^2 \\
        &= \left|\left|\sum_{y \in \mc{I}^{x_*}} \frac{\alpha_y (x_*-y)}{\sum_{x \in \mc{I}^{x_*}} \alpha_x (x_*-x)^{\top}\theta}\right|\right|_{M}^2 \\
        &= \left|\left|\sum_{y \in \mc{I}^{x_*}} \frac{\alpha_y (x_*-y)^\top\theta}{\sum_{x \in \mc{I}^{x_*}} \alpha_x (x_*-x)^{\top}\theta} \cdot \frac{ (x_*-y)}{(x_*-y)^\top\theta}\right|\right|_{M}^2 \\
        &\leq \sum_{y \in \mc{I}^{x_*}}\frac{\alpha_y (x_*-y)^\top\theta}{\sum_{x \in \mc{I}^{x_*}} \alpha_x (x_*-x)^{\top}\theta} \cdot \left|\left|  \frac{ (x_*-y)}{(x_*-y)^\top\theta}\right|\right|_{M}^2 && \text{(Jensen's inequality)} \\
        &\leq \sum_{y \in \mc{I}^{x_*}}\frac{\alpha_y (x_*-y)^\top\theta}{\sum_{x \in \mc{I}^{x_*}} \alpha_x (x_*-x)^{\top}\theta}\cdot\max_{x' \in \mc{I}^{x_*}}\left|\left|  \frac{ (x_*-x')}{(x_*-x')^\top\theta}\right|\right|_{M}^2 \\
        &= \max_{x' \in \mc{I}^{x_*}}\left|\left|  \frac{ (x_*-x')}{(x_*-x')^\top\theta}\right|\right|_{M}^2 \cdot \sum_{y \in \mc{I}^{x_*}}\frac{\alpha_y (x_*-y)^\top\theta}{\sum_{x \in \mc{I}^{x_*}} \alpha_x (x_*-x)^{\top}\theta}\\
        &= \max_{x' \in \mc{I}^{x_*}}\left|\left|  \frac{ (x_*-x')}{(x_*-x')^\top\theta}\right|\right|_{M}^2.
    \end{align*}
\end{proof}

\end{document}